\documentclass[letterpaper]{article} % DO NOT CHANGE THIS
\usepackage{aaai24}  % DO NOT CHANGE THIS
\usepackage{times}  % DO NOT CHANGE THIS
\usepackage{helvet}  % DO NOT CHANGE THIS
\usepackage{courier}  % DO NOT CHANGE THIS
\usepackage[hyphens]{url}  % DO NOT CHANGE THIS
\usepackage{graphicx} % DO NOT CHANGE THIS
\usepackage{natbib}  % DO NOT CHANGE THIS AND DO NOT ADD ANY OPTIONS TO IT
\usepackage{caption} % DO NOT CHANGE THIS AND DO NOT ADD ANY OPTIONS TO IT
\usepackage{algorithm}
\usepackage{algpseudocode}
\usepackage{amsmath,amsthm}
\usepackage{amssymb}
\usepackage{subfig}
\usepackage{adjustbox}
\usepackage{booktabs}
\usepackage{epsfig}
\usepackage{color}

\usepackage{newfloat}
\usepackage{listings}
\DeclareCaptionStyle{ruled}{labelfont=normalfont,labelsep=colon,strut=off} % DO NOT CHANGE THIS
\floatstyle{ruled}
\newfloat{listing}{tb}{lst}{}
\floatname{listing}{Listing}
\newcommand{\E}{\mathbb{E}}
\newcommand{\R}{\mathbb{R}}
\newcommand{\eg}{e.g.}
\theoremstyle{plain}
\newtheorem*{theorem*}{Theorem}
\newtheorem{theorem}{Theorem}[section]

\newtheorem{lemma}[theorem]{Lemma}
\newtheorem{corollary}[theorem]{Corollary}
\theoremstyle{definition}
\newtheorem{definition}[theorem]{Definition}

\newtheorem{remark}[theorem]{Remark}

\title{Generalizing across Temporal Domains with Koopman Operators}
\author{
    Qiuhao Zeng\textsuperscript{\rm 1}, Wei Wang\textsuperscript{\rm 1}, Fan Zhou\textsuperscript{\rm 2}, Gezheng Xu\textsuperscript{\rm 1}, Ruizhi Pu\textsuperscript{\rm 1}, Changjian Shui\textsuperscript{\rm 3},\\
    Christian Gagné\textsuperscript{\rm 4}, Shichun Yang\textsuperscript{\rm 2}, Boyu Wang\textsuperscript{\rm 1, 3}\thanks{Corresponding authors: Fan Zhou, Charles X. Ling, Boyu Wang.}, Charles X. Ling\textsuperscript{\rm 1}
}
\affiliations{
    \textsuperscript{\rm 1}The University of Western Ontario \\
    \textsuperscript{\rm 2}Beihang University \\
    \textsuperscript{\rm 3}Vector Institute \\
    \textsuperscript{\rm 4}Université Laval \\
}

\usepackage{bibentry}
\begin{document}

\maketitle

\begin{abstract}
In the field of domain generalization, the task of constructing a predictive model capable of generalizing to a target domain without access to target data remains challenging. This problem becomes further complicated when considering evolving dynamics between domains. While various approaches have been proposed to address this issue, a comprehensive understanding of the underlying generalization theory is still lacking. In this study, we contribute novel theoretic results that aligning conditional distribution leads to the reduction of generalization bounds. Our analysis serves as a key motivation for solving the Temporal Domain Generalization (TDG) problem through the application of Koopman Neural Operators, resulting in Temporal Koopman Networks (TKNets). By employing Koopman Operators, we effectively address the time-evolving distributions encountered in TDG using the principles of Koopman theory, where measurement functions are sought to establish linear transition relations between evolving domains. Through empirical evaluations conducted on synthetic and real-world datasets, we validate the effectiveness of our proposed approach.
\end{abstract}

%%%%%%%%% BODY TEXT
\section{Introduction}
\label{sec:intro}

Modern machine learning techniques have achieved unprecedented success over the past decades in numerous areas.
However, one fundamental limitation of most existing techniques is that a model trained on one dataset cannot generalize well on another dataset if it is sampled from a different distribution. Domain generalization (DG) aims to alleviate the prediction gap between the observed source domains and an \emph{unseen} target domain by leveraging the knowledge extracted from multiple source domains~\cite{wjd_24, multi2, wjd_88,li2018learning}. 

% To improve the generalization ability in the target domain, one promising solution to DG is to learn a domain-invariant feature representation [??]  

Existing DG methods can be roughly categorized into three groups: data augmentation / generation, disentangled / domain-invariant feature learning, and meta-learning \cite{jd_survey}. In many real-world applications, the temporal dynamics across domains are common and can be leveraged to improve accuracy for the unseen target domain \cite{usgda, laed, cida}. However, one intrinsic problem with these existing DG methods is that most of them treat all the domains equally and ignore the relationship between them, implicitly assuming that they are all sampled from a stationary environment.
% However, in many real-world applications, the data are usually collected sequentially and the learning tasks can vary in an \emph{evolving} manner. 
% the data are usually collected sequentially and the learning tasks can vary in an \emph{evolving} manner. 
% For example, geological exploration is often carried out periodically and the distribution of data collected can change from year to year due to environmental changes. Medical data are also often collected with age or other indicators as intervals, and there is an evolving trend in the data of different groups.
For example, it is common that source domains are constituted of images collected over the last few years and the target domain is the unseen future. 
For geological applications, the source samples can be collected along different altitudes, longitude, and latitude, while the target is to generalize to some regions where the data is absent due to inaccessibility.
Medical data is also often collected with age or other indicators as intervals, and we hope the model can perform well on younger or elder age groups where the samples may be rare.
As a more concrete example, Fig.~1(a) shows several instances from the rotated MNIST (RMNIST) dataset, a widely used benchmark in the DG literature, where the digit images of each subsequent domain are rotated by $15^\circ$. Fig.~1(b) reports the generalization performances of several state-of-the-art DG algorithms on the data set, from which it can be clearly observed that the performances drop when deploying the models on outer domains (i.e., 
domains of 0 and 75 degrees). The results indicate that the algorithms ignore the evolving pattern between the domains. Consequently, they are good at ``interpolation" but not at ``extrapolation". 
\begin{figure*}[!tb]
		\subfloat[]{\includegraphics[width=0.33\textwidth]{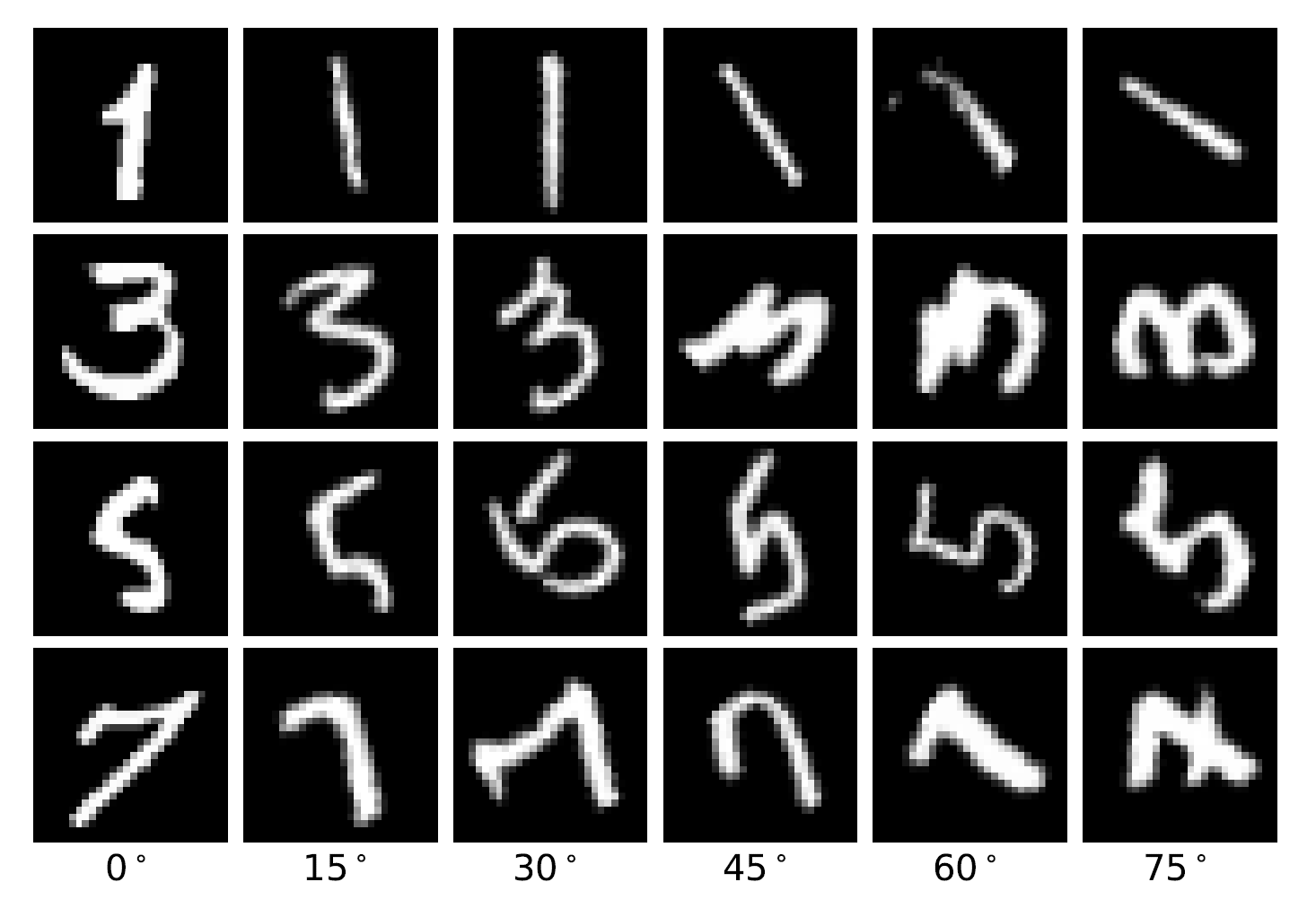}\label{motivation_figure:rmnist}}
		\subfloat[]{\includegraphics[width=0.33\textwidth]{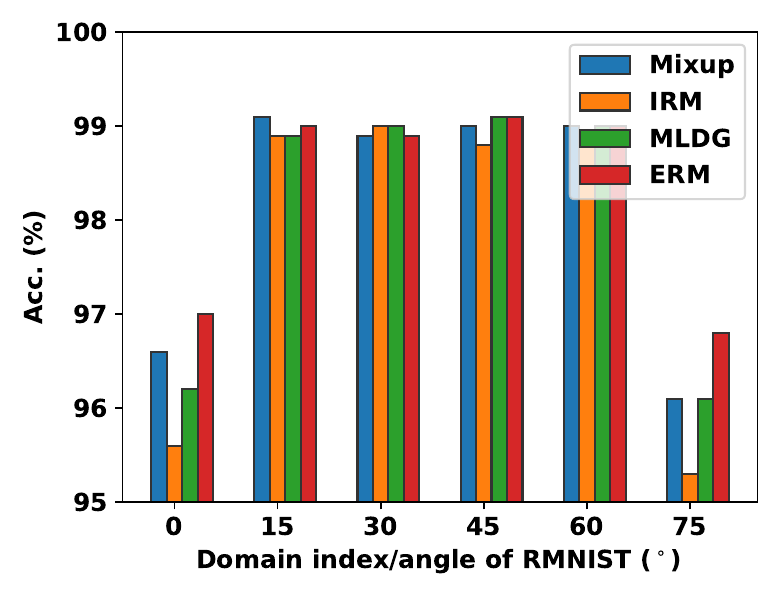}\label{motivation_figure:acc_baselines}}
		\subfloat[]{\includegraphics[width=0.33\textwidth]{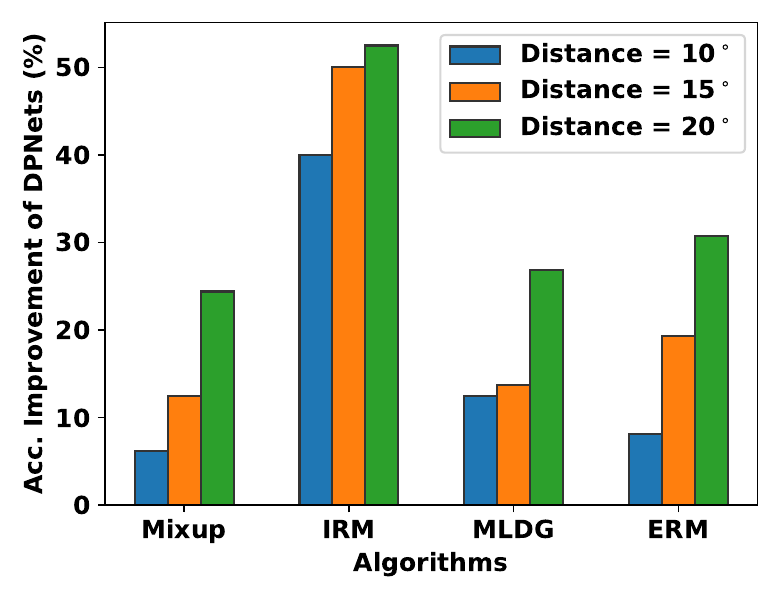}\label{motivation_figure:acc_forward_backward}}
\caption{(a) Evolving manner among RMNIST domains. (b) Accuracy of traditional DG methods on evolving domains. These methods cannot generalize well on outer domains ($0^\circ$ and $75^\circ$). (c) Comparison between the performance of our method and baselines on outer domains. The proposed method outperforms all the baselines.}
\label{motivation_figure}
\end{figure*}

In this paper, we address this learning scenario as \emph{temporal domain generalization} (TDG)~\cite{zeng2023foresee,bai2022temporal,gi,lssae}, which aims to capture and exploit the temporal dynamics in the environment. TDG aims to generalize to a target domain along a specific direction by extracting and leveraging the relations between source domains. Specifically, we develop a novel theoretical analysis that highlights the importance of modeling the relation between two consecutive domains to extract the evolving pattern of the environment. 
% Moreover, our analysis also suggests learning a globally consistent directional mapping function via meta-learning.
Koopman theory \cite{koopman1931hamiltonian} states that any complex dynamics can be modeled by a linear Koopman operator acting on the space of measurement functions. Inspired by our theoretical results, we propose to capture the temporal dynamics using the Koopman operator and align the distribution of data in the Koopman space. As a comparison, Fig. 1(c) shows the performance improvement of TKNets over the other algorithms on RMNIST dataset. It can be observed that the performance gap between TKNets and the other baselines has widened as the domain distance increases. More details can be found in Sec.~\ref{sec:exp}.

Here, we emphasize the key difference between Temporal Domain Genralization~\cite{bai2022temporal,gi,lssae} and Temporal Domain Adaptation \cite{cudaal,laed,cida}. While both learning paradigms aim to tackle the issue of evolving domain shifts, the latter still requires unlabeled instances from the target domain. In this sense, TDG is more challenging, and existing theoretical and algorithmic results cannot be applied to this problem directly. 

We summarize our \textbf{contributions} as three-fold: (1) We derive theoretical analysis to the generalization bound of the TDG problem, which highlights the importance of learning the transition function to mitigate the temporal domain shifts; (2) Motivated by the theoretic results, we propose a novel algorithm TKNets, which learns the complex and nonlinear dynamics based on Koopman theory; (3) We conduct experiments on synthetic and real-world datasets, and the empirical results suggest an improved performance.
% in perdicting the test environment.
  % As far as we know, we are the first to bridge the gap between DPNets Protypical Networks~\cite{proto} and the theoretic bound
% \section{Introduction}
% In most scenarios of machine learning, the environments for training and testing are different. Aware of this problem, the research community has spent significant effort to develop algorithms which can generalize across all domains, which is called domain generalization(DG)[citations]. 

% \paragraph{Our contribution.}  (1) We identify the domain generalization problem in an evolving environment and proposed evolving domain generalization (EDG) as the first DG framework for addressing this problem.  (2) We develop a theoretical understanding of EDG, which highlights the role of learning a mapping function to capture the evolving pattern over domains. (3) Based on our analysis, we propose directional prototypical networks (DPNets), a simple and efficient algorithm for EDG.

%------------------------------------------------------------------------

\section{Related Work}

\paragraph{Domain Generalization (DG).} Domain generalization aims to train a model which generalizes on all domains. Existing DG methods can be classified into three categories. The first and most popular category is representation learning, which focuses on learning a common representation across domains. It can be achieved by domain-invariant representation learning \cite{wjd_24,wjd_59,wjd_77,wjd_88} and feature disentanglement \cite{wjd_93, wjd_101}. The former focuses on aligning latent features across domains, and the latter tries to distill domain-shared features. Secondly, data manipulation can also empower the model with generalization capability. Data manipulating techniques include data augmentation \cite{wjd_28, shankar2018generalizing}, which usually extends the dataset by applying specific transformations on existing samples, and data generation \cite{wjd_38}, which often applies neural networks to generate new samples.
\cite{dirt} converts DG to an infinite-dimensional constrained statistical learning problem under a natural model of data generation.
The theoretically grounded method proposed in \cite{dirt} leverages generating model among domains to learn domain-invariant representation. 
The third most commonly used method is meta-learning. The meta-learning framework \cite{li2018learning, wjd_13, wjd_14} is used to improve the generalizing capability by simulating the shift among domains. 
Apart from the above three categories, \cite{wjd_104} tries to ensemble multiple models into a unified one that can generalize across domains. 
The DRO-based methods \cite{rahimian2019distributionally} which aim to learn a model at worst-case distribution scenario also match the target of DG well. Besides, gradient operation \cite{wjd_112}, self-supervision \cite{wjd_114} and random forest \cite{wjd_115} are also exploited to improve generalizing capability.
Different from the existing DG methods that focus on learning one unified model for all domains, our approach tries to train a prediction model for the target domain by leveraging the evolving pattern among domains.
    \paragraph{Temporal Domain Adaptation (TDA) / Temporal Domain Generalization (TDG).} Many previous works in the domain adaptation area notice the evolving pattern of domains and leverage it to improve performance in different settings. \cite{laed} proposes a meta-adaptation framework that enables the learner to adapt from one single source domain to continually evolving target domains without forgetting. \cite{usgda} focuses on adapting from source domains to the target domain with large shifts by leveraging the unlabeled intermediate samples. \cite{cida} combines the traditional adversarial adaptation strategy with a novel regression discriminator that models the encoding-conditioned domain index distribution. \cite{chen2021gradual} investigate how to discover the sequence of intermediate domains without index information and then adapt to the final target. 
    % The experimental results and theoretical analysis demonstrate the value of leveraging index information when working on evolving domains. These studies fully demonstrate that leveraging evolving patterns between domains is beneficial and worth more exploration. However, there are two significant limitations in previous works. The first one is the requirements for accessing unlabeled data in the target domain. The second one is that all previous theoretical results are based on the assumption that the distance between sequential domains is small, which would become vacuous as the environment evolves and is contrary to the fact that more domains provide more evolving information which can help to improve the performance. 
    TDG recently has been actively studied. \cite{lssae} introduce a probabilistic model to mitigate the covariate shift and concept shift. \cite{gi} design gradient interpolation (GI) loss to penalize the curvature of the learned model along time and design TReLU activation whose parameters are tuned w.r.t. time. \cite{bai2022temporal} proposed a novel Bayesian framework to explicitly model the concept drift over time. \cite{zeng2023foresee} proposed to learn the evolving patterns in the framework of meta-learning. However, existing TDG methods have not explored the theoretic generalization bound of TDG problem. In this paper, our theoretical results are based on proposed $\lambda$-consistency, an intuitive and realistic measurement of evolving levels in the environments.
    
%------------------------------------------------------------------------
\section{Background}
%\subsection{Notations}

\subsection{Koopman Theory}
In discrete-time dynamical systems, the formulation of state transitions is as $v_{t+1} = F(v_t)$, where $v\in \mathcal{V}\subseteq\mathbb{R}^d$ represents the system state and $F$ characterizes the vector field governing the system dynamics. However, it's difficult to directly capture the system dynamics due to the presence of nonlinearity or noisy data. To address this concern, Koopman theory \cite{koopman1931hamiltonian} hypothesizes that the system state can be effectively projected onto an infinite-dimensional Hilbert space defined by measurement function $\mathcal{G}(\mathcal{V}):=\{g:\mathcal{G}\rightarrow \mathbb{R}\}$. This projected space can then be advanced forward in time through an infinite-dimensional linear operator $\mathcal{K}$, hence
\begin{equation}
    \mathcal{K}\circ g(v_t) = g(F(v_t)) = g(v_{t+1}).
\end{equation}
The Koopman operator facilitates the transition of observations of the state to subsequent time steps by mapping between function spaces~\cite{li2020learning}. In this work, we employ measurement functions to effectively map encoded features to the functional space. By adopting Koopman theory, we can learn the transition model, which characterizes the temporal dynamics governing the system's behavior.

\subsection{Temporal Domain Generalization}
We aim to propose a formal understanding of the generalization bound in predicting the unseen target environment. Let $\{\mathcal D_1(z,y),\mathcal D_2(z,y),...,\mathcal D_m(z,y)\}$ be $m$ observed source domains sampled from an environment $\mathcal{E}$, 
%and $\mathcal S_i=\{(x^{(i)}_j,y^{(i)}_j)\}_{j=1}^{n_i}$ be a set of $n_i$ identically and independently distributed (i.i.d) instances sampled from the $\mathcal D_i$, 
where $z\in \mathcal{Z}$ and $y \in \mathcal{Y}$ are, respectively, the data Koopman embedding and its corresponding label, and $\mathcal{D}_i(z,y)$ characterizes the joint probability over the $i$-th domain. The embeddings $z$ are mapped by a composite embedding function such that $\mathcal{G}\circ\phi(\mathcal{X})\rightarrow\mathcal{Z}$, where $x\in \mathcal{X}$ is the input raw data, $\phi$ is a normal embedding neural network, and $\mathcal{G}$ is the measurement function. The goal of TDG is to learn a hypothesis $h\in \mathcal{H}$ so that it can have a low risk on an unseen but time-evolving target domain $\mathcal{D}_t$:
\begin{align*}
    R_{\mathcal{D}_t}(h) \triangleq \mathbb E_{(z,y)\sim \mathcal D_t}[\ell(h(z),y)]
\end{align*}

% \textcolor{red}{$\mathcal{E}(\mathcal{D}_i)=\mathcal{P}(\mathcal{D}_i|\mathcal{D}_{i-1})$}

% domains evolving sequentially and $S_i=\{(x^{(i)}_j,y^{(i)}_j)\}_{j=1}^{n_i}$ be a set of $n_i$ identically and independently distributed (i.i.d) instances sampled from the $\mathcal D_i$. These tasks evolve sequentially and can be mathematically modeled as the existence of a transforming function $\overline g$ such that  $\mathcal D_{i+1}=\overline g(\mathcal D_i)$ holds for any adjacent domains in$\{\mathcal D_1,\mathcal D_2,...,\mathcal D_m,\mathcal D_t\}$. Finally, the goal of EDG is to find $h\in \mathcal H$ which minimize the following objective function:

% $$\mathbb E_{(x,y)\sim \mathcal D^t}[\ell(h(x),y)]$$

\noindent where $\ell:\mathcal Y\times\mathcal Y\rightarrow \mathbb R_+$ is a non-negative loss function, and $\mathcal{H}$ is a hypothesis class that maps $\mathcal Z$ to the set $\mathcal Y$. 
% In the setting of traditional DG, as there is no relation exploited between the source domains and the target domain, most existing techniques essentially either ``enlarge" the input space  $\mathcal{Z}$ along all possible directions~\cite{shui2020beyond,volpi2018generalizing,shankar2018generalizing,qiao2020learning} or learn a domain-invariant feature representation via domain alignment~\cite{li2018domain,wjd_88,zhao2020domain}. In contrast, 
The objective of TDG is to generalize the model on $\mathcal{D}_t$ along a specific direction when there is an underlying evolving pattern between the source domains and $\mathcal{D}_t=\mathcal{D}_{m+1}$. 

\section{Theoretical Analysis and Methodology}
\subsection{Theoretical Motivations} \label{assumption_comparison}

    % In this section, we first address the limitation of previous theoretical framework of evolving environment and then introduce the motivation, definition and benefits of our definition of $\lambda$-consistency.
    % The most common theoretical framework of DG and DA is to bound the distance between source and target. Although with some differences, almost all previous 
Existing theoretical studies of temporal domain \cite{cida, usgda, laed, hanzhao_gda} make assumptions that "\textbf{consecutive domains are similar}", where the similarity is characterized by some assumption on a distributional distance $d(\mathcal D_t,\mathcal D_{t+1})$. For instance, \cite{usgda,lssae} assume that $d(\mathcal D_t,\mathcal D_{t+1})<\epsilon$, and the assumption in \cite{laed} is $d_{\mathcal H\Delta\mathcal H}(\mathcal D_{t_1},\mathcal D_{t_2})\leq \alpha|t_1-t_2|$. While such assumptions and the corresponding theoretical frameworks are intuitive, 
    % we argue that they have the following limitations when applied to the EDG scenario: 
    % Then, based on the fact that number of domains $T$ is also limited, the cumulative distance bound of source and target with format similar to $\mathcal O(T\cdot d)$ can be derived. This is an intuitive theoretical framework while we argue that applying it on evolving environment problem is not proper:
      % (1) The similarity assumption is too restrictive in evolving environment.
      %   The environment evolving does not indicate that the distance between two consecutive domains is small. For example, a tiny angular image rotation on images transforming $\mathcal D_t$ to $\mathcal D_{t+1}$ can result in a large distance of distribution shift for each pixel. Due to the cumulative effect, expecting the distance between the source and the target to be small is also unrealistic.  
      the similarity assumption does not properly characterize the temporal dynamics. Instead, we argue that we should focus on \emph{consistency} of the environment instead of \emph{similarity} of consecutive domains. With the Koopman theory applied, we ensure that consistency holds between the temporal domains~\cite{azencot2020forecasting}. 

To capture the consistent evolving pattern in the Koopman space of $\mathcal{E}$, it is reasonable to assume that such a pattern can be \emph{approximately} captured by the governing function $\mathcal{K}: \mathcal{Z} \rightarrow \mathcal{Z}$
% \footnote{The mapping function $g$ can either be  deterministic~\cite{tachet2020domain} or probabilistic~\cite{shui2021aggregating}. Empirically, the deterministic class works better for our method. } 
% g(\mathcal{D}_i) =
in a way such that the \emph{forecasted} domain $\mathcal{D}_{i+1}^\mathcal{K} \triangleq  \mathcal{D}_i(\hat{z},y)$ is close to $\mathcal{D}_{i+1}$ as much as possible, where $\hat{z}=\mathcal{K} \circ \mathcal{G}\circ\phi(x)$. Then, the forecasted domain $\mathcal{D}^\mathcal{K}_t$ (transformed from the last source domain $\mathcal{D}_m$) can be adopted as an alternative to $\mathcal{D}_t$ for domain generalization if $\mathcal{K}$ can properly capture the evolving pattern. Intuitively, capturing the evolving pattern in $\mathcal{E}$ is hopeless if it varies arbitrarily. 
On the other hand, if the underlying pattern is consistent over domains in the Koopman space, it is reasonable to assume that there exists a Koopman operator $\mathcal{K}^*$ that would perform consistently well over all the domain pairs. 
% For example, given numbers $100, 202, 301$, one would expect that the numbers increase by around 100 and the next number will be around $400$, but it is challenging to guess the number if the first three numbers are $-20, 1300, 4$.  
% \icmlbb{Intuitively, capturing the evolving pattern is promising only if the pattern is consistent.
% For example, imagine a self-driving car is driving on the highway when it is at dusk, and the surrounding environment is getting darker consistently, then predicting the future light conditions is promising. Conversely, if the surrounding light conditions vary arbitrarily over time, then it becomes almost impossible to model the evolving process.
To this end, we introduce the notion of \emph{$\lambda$-consistency} of an environment $\mathcal{E}$.
% More specifically, Lemma~\ref{lemmastart} bounds the risk on the target domain $\mathcal{D}_t$ with respect to $\mathcal{D}_t^g$.

% The above rules lead to our definition of \emph{consistency} of an environment $\mathcal{E}$.  The "similarity" between $\{g_i\}_{i=1}^m$ can characterize the consistency of evolving since they describe each evolving step of the environment. A basic solution is calculating average distance between pairs of $\{g_i\}_{i=1}^m$ in the function space. While the first problem is that we do not know ground-truth $\{g_i\}_{i=1}^m$ and the distance in function space is not realistic. So we turn to use the distance of generated sample by applying two different functions as their distance, which is more realistic. Secondly, if the underline $\{g_i\}_{i=1}^m$ is "similar" while it can not be modeled by neural network, then the similarity is still helpless. So we propose to measure the consistency by if there exist one $g$ in function space $\mathcal G$ which is close to all $\{g_i\}_{i=1}^m$. Then we naturally switch to the following definition:

\begin{definition}[$\lambda$-Consistency] 
\label{def1}
We define the ideal mapping function in the worst-case domain: $\mathcal{K}^* = \arg\min_{\mathcal{K}} \max_{\mathcal{D}_i\in\mathcal{E}} d_{}(\mathcal{D}_i||\mathcal{D}^\mathcal{K}_i)$, where $\mathcal{D}_{i+1}^\mathcal{K} \triangleq \mathcal{D}_i(\hat{z},y)$. Then, an evolving environment $\mathcal{E}$ is \emph{$\lambda$-consistent} if the following holds:
% Let $d_{trans}(g,g'):=\sup_{\mathcal D \in \mathcal E} d_{}(\mathcal D^g||\mathcal D^{g'})$ be the distance between to transforming function. Then, the evolving environment $\mathcal E$ is \emph{$\lambda$-consistent} if the following holds: 
\begin{align*}
    |d_{}(\mathcal D_i||\mathcal D^{\mathcal{K}^*}_{i})-d_{}(\mathcal D_j||\mathcal D^{\mathcal{K}^*}_{j})| \le \lambda, \qquad \forall \mathcal{D}_i, \mathcal{D}_j \in \mathcal{E}.
    % \sup_{g,g'\in{\{g_2,...,g_{n}\}}} d_{trans}(g,g')= \lambda
\end{align*}
\end{definition}
\noindent where $d_{}(\mathcal D_t||\mathcal D^\mathcal{K}_t)$ is the Kullback–Leibler (KL) divergence between $\mathcal D^\mathcal{K}_t$ and $\mathcal D_t$~\cite{jsd}.

Our definition of $\lambda$-consistency offers a fundamentally new perspective to depict the temporal pattern. The distance is between forecasted domains and real domains, without any assumptions about the distance of domain pairs in $\mathcal E$. 
It is worth mentioning that $\lambda$-consistency is not an assumption. Instead, it is only a value that depicts how \emph{consistently} environment evolves in $\mathcal{E}$. 
A small $\lambda$ means that the environment evolves steadily, which means the future domains are predictable and can be leveraged for performance improvement. 
Reversely, there are some cases in which the evolving mechanism is too complex, non-Markov, or too random, which will result in a large $\lambda$. 

We show the generalization bound between $R_{\mathcal D_t}(h)$ and $R_{\mathcal D_t^{\mathcal{K}^*}}(h)$ as follows, the complete proof is deferred to the Section \ref{appendix_theory} in Appendix due to space limitations:

\begin{theorem}
\label{theoremds}
Let $\{\mathcal D_1,\mathcal D_2,...,\mathcal D_m\}$ be $m$ observed source domains sampled sequentially from an evolving environment $\mathcal{E}$, and $\mathcal{D}_t$ be the next unseen target domain: $\mathcal{D}_t = \mathcal{D}_{m+1}$. $G$ is the range of the interval of the loss function. Then, if $\mathcal{E}$ is $\lambda$-consistent, we have
{\small
\begin{align} 
 &R_{\mathcal D_t}(h) \leq R_{\mathcal D_t^{\mathcal{K}^*}}(h)+\frac{G}{\sqrt{2(m-1)}} \times\nonumber\\
 &\Bigg({\sqrt{\sum _{i=2}^{m}\mathbb E_{y\sim \mathcal D_{i}(y)}d_{}(\mathcal D_{i}(z|y)||\mathcal D^{\mathcal{K}^*}_{i}(z|y))}}+\sqrt{(m-1)\lambda}\Bigg)
 \label{eqn:bound}
\end{align}
}
\end{theorem}

% We note three key theoretical differences between EDG and previous studies of DA in evolving environments \citep{david,laed}. 
% (1) The distance is between synthetic domain and real domain, without any unrealistic assumptions about distance of domain pairs in $\mathcal E$. 
% (2) As we are not target on bound source and target distance, the bound would not become loose as $T$ grows.
% (3) The bound is strongly correlated with the condition of evolving pattern in the environment, indicating under which case leveraging evolving pattern is achievable.
\noindent\textbf{Discussion} We note two key theoretical differences between TDG and previous studies of DA in evolving environments \cite{david,laed}.
(1)~In Domain Adaptation (DA), the target risk is bounded in terms of source and target domains (\eg, $\mathcal{H}\Delta \mathcal{H}$-divergence), while our analysis relies on the distance between forecasted and real domains. (2)~DA theories are built upon the assumption that there exists an ideal joint hypothesis that achieves a low combined error on both domains, while our assumption is the $\lambda$-consistency in the Koopman space of $\mathcal{E}$. The theoretical result motivates us to design algorithms with Koopman operators, which directly minimizes the KL-term in Eq.~(\ref{eqn:bound}). We further bridge the gap between the theory and the method in Section~\ref{sec:theory-to-implement}.

\subsection{Proposed Methods}\label{sec:algo}
We propose Temporal Koopman Networks (TKNets), a deep neural network model based on Koopman theory to capture the evolving patterns between the Temporal Domains. 
% Our theoretical analysis indicate that there are two important guidelines for solving evolving domains.
Our analysis reveals two strategies for designing TKNets:
\begin{itemize}
    \item[(i)] Learning the Koopman operator $\mathcal{K}$ to capture the evolving pattern by minimizing the distance between the distributions of forecasted and real domains in the Koopman space.  
    % \vspace{-4pt}
    % \item[(ii)] Learning $g$ and $h$ to minimize the risk on the synthetic target domain $\mathcal{D}_t^{g}$. 
    % \vspace{-4pt}
    \item[(ii)]  Note that {\small $\mathcal D^{\mathcal{K} }_{i+1} = \mathcal{K} \circ \mathcal{G} \circ \phi(\mathcal{D}_i)$} is produced from $\mathcal{D}_i$, but its quality is evaluated on $\mathcal{D}_{i+1}$. Minimizing {\small $d_{}(\mathcal D^{\mathcal{K} }_{i}||\mathcal D_{i})$} naturally leads to the adoption of the Koopman Theory for learning the linear infinite-dimensional operator $\mathcal{K} $.
\end{itemize}

Specifically, the input samples from both $\mathcal D_i$ and $\mathcal D_{i+1}$ are first transformed to a representation space $\mathcal{V}$ induced by an embedding function $\phi: \mathcal X\rightarrow\mathcal V$. Then we use measurement functions $\mathcal{G}:=[g_1,\dots,g_n]$ that span the Koopman space function $\mathcal{G}: \mathcal{V} \rightarrow \mathcal{Z}$ (for example, the first dimension of measurement functions can be $g_1(x)=\sin{x}$) which maps the samples into the Koopman space so that we can find a Koopman operator $\mathcal{K} $ leading to distribution of $\mathcal{K} \circ \mathcal{G} \circ\phi(\mathcal{D}_i)$ aligned with the distribution of $\mathcal{G}\circ\phi(\mathcal{D}_{i+1})$. 
% For simplicity, let $f_\phi = g \circ f_{\psi}$ be the composite function parameterized by $\phi$. 
In all, TKNets consists of three components: a embedding function $\phi: \mathcal{X}\rightarrow{V}$, measurement functions $\mathcal{G} :\mathcal{V}\rightarrow Z$, and Koopman operator $\mathcal{K} :\mathcal{G}(\mathcal{V}) \rightarrow \mathcal{G}(\mathcal{V})$. The key idea of TKNets is 
% Specifically, the map function DPNets Compared with \cite{proto}, the difference of DPNets is that its mapping function $g$  consists of two different embedding functions: $f_{\theta, \phi}$ for $\mathcal{D}_i$ and $f_{\theta}$ for $\mathcal{D}_{i+1}$, where $f_{\theta, \phi}=f_{\phi}\circ f_{\theta}$, $\theta$ and $\phi$ are learnable parameters. 
to use $\{\phi, \mathcal{G} , \mathcal{K} \}$ to capture the evolving pattern of $\mathcal{E}$ by mapping the data in the nonlinear dynamic system into the Koopman space such that the transition function can be modeled by an infinite-dimensional linear Koopman operator. We actually align the distribution of $\mathcal{D}^{\mathcal{K} }_{i+1}(z|y)$ and $\mathcal{D}_{i+1}(z|y)$, as suggested by Theorem~\ref{theoremds}. We implicitly minimize the distance between class-conditional semantic centroids \cite{xie2018learning,shui2021aggregating}, which is an approximation of the semantic conditional distribution of each class. We illustrate the relationship between theory and practical implementation in Sec. \ref{sec:theory-to-implement}.

% the prototypes of $f_\psi(\mathcal{D}_{i+1})$ using $f_\phi(\mathcal{D}_i)$. 
% As each prototype can be viewed as the centroid of instances of each class, which , DPNets essentially minimizes the distance between 
% \icml{As in previous studies \citep{mbdg, dirt}, we slightly abuse notations and use $g$ to denote instance-wise transforming as implementation of domain transforming $g$ mentioned in theoretical section.}

Let $\mathcal{S}_i = \{(x_n^i,y_n^i)\}_{n=1}^{N_i}$ be the data set of size $N_i$ sampled from $\mathcal{D}_i$, and $\mathcal{S}_i^k$ be the subset of $\mathcal S_i$ with class $k\in\{1,...,K\}$, where $K$ is the total number of classes. In TKNets, the semantic centroids are computed from the support set $\mathcal{S}_i$ through the composite function $\mathcal{K}\circ\mathcal{G}\circ \phi$. Since the Koopman operator captures the transition relationship between two consecutive domains, the computed centroids from $\mathcal{S}_i$ are the forecasted centroids in the domain $i+1$. Then, the forecasted centroid of domain $i+1$ is the mean vector of the \emph{support} instances belonging to $\mathcal{S}_i^k$:
{\small
\begin{align*}
    c_{i+1}^k=\frac{1}{|\mathcal{S}^k_i|}\sum_{(x_n^i,y_n^i)\in \mathcal{S}^k_i}\mathcal{K}\circ\mathcal{G}\circ  \phi(x_n^i)
\end{align*}
}

The query instances are from $\mathcal{S}_{i+1}$ and are passed through the composite embedding function $\mathcal{G}\circ  \phi$. The predictive distribution for $\mathcal{D}_{i+1}$ is given by 
{\small
\begin{align}\label{eq1}
    \mathcal{D}(y^{i+1}=k|x^{i+1})=\frac{\exp(-d_{\text{eu}}(\mathcal{G}\circ \phi(x^{i+1}), c^k_i))}{\sum_{k'=1}^{K}\exp(-d_{\text{eu}}(\mathcal{G}\circ \phi(x^{i+1}), c^{k'}_i))},
\end{align}}where $d_{\text{eu}}: \mathcal{Z}\times\mathcal{Z}\rightarrow [0,+\infty)$ is a euclidean distance function of embedding space, and we adopt squared Euclidean distance in our implementation, as suggested in~\cite{proto}. During the training stage, at each step, we randomly choose the data sets $\mathcal S_i,\mathcal S_{i+1}$ from two consecutive domains as support and query sets, respectively. Then, we sample $N_B$ samples from each class $k$ in $\mathcal S_i$, which is used to compute centroid $c_i^k$ for the query data in $\mathcal{S}_{i+1}$.  Model optimization proceeds by minimizing the negative log probability: 
\begin{equation}
    \label{eqn:loss}
    J_{}=\sum_{i=1}^{m-1}\sum_{k=1}^K\sum_{n=1}^{N_B}-\frac{1}{KN_B}\log \mathcal{D}_{}(y_{n}^{i+1}=k|x_n^{i+1})
\end{equation}

The pseudocode to compute $J$  for a training episode is shown in Algorithm~\ref{alg:cap}. In the testing stage, we pass the instances from $\mathcal S_m$ and $\mathcal S_t$ through $\mathcal{K}\circ\mathcal{G}\circ  \phi$ and $\mathcal{G}\circ  \phi$ respectively as support and query sets and then make predictions for the instances in $\mathcal{D}_t$ using Eq.~(\ref{eq1}).

\begin{algorithm}[tb]
    \caption{TKNets (one episode)}
    \label{alg:cap}
\begin{algorithmic}
    \State \textbf{Input: }  $\{\mathcal S_1,\mathcal S_2,...,\mathcal S_m\}$: $m$ data sets from consecutive domains. $N_B$: the number of support and query instances for each class. $\textsc{RandomSample}(\mathcal S, N)$: a set of $N$ instances sampled uniformly from the set $\mathcal S$ without replacement. 
    % \STATE \textbf{Output: } The loss $J_{\phi,\psi}$ and $J_{\omega}$ for a randomly generated training episode.
    \State \textbf{Output: } The loss $J$ for a randomly generated training episode.
    \State $t\leftarrow \textsc{RandomSample}(\{1,...,m\})$ 
    \For{$k$ in $\{1,...,K\}$}
        \State $S^k\leftarrow \textsc{RandomSample}(\mathcal S_{i}^{k},N_B)$
        % \Comment{average sampling for label shift}
        \State $c^k=\frac{1}{|S^k|}\sum_{(x_j,y_j)\in S^k}\mathcal{K}\circ\mathcal{G}\circ \phi(x_j)$ 
        \State $Q^k\leftarrow \textsc{RandomSample}(\mathcal S_{i+1}^k,N_B)$ 
    \EndFor
    \State $J\leftarrow0$
    \For{$k$ in $\{1,...,K\}$}
        \For {$(x,y)$ in $Q$}
            \State $J\leftarrow J+\frac{1}{KN_B}\Big [d(\mathcal{G}\circ \phi(x),c^k)$\\           
               \State $\qquad\qquad+\log\sum_{k'}\exp(-d(\mathcal{G}\circ  \phi(x),c^k))\Big ]$
        \EndFor
        % \icml{
        % % \STATE $J_{\phi,\psi}\leftarrow J_{\phi,\psi}+\lambda_{JS}\Big [\sum_{(x_j,y_j)\in S^k}\log(1-f_\omega(f_\psi(x_j))) + \sum_{(x_j,y_j)\in Q^k}\log(f_\omega(f_\phi(x_j)))\Big ]$
        % % \STATE $J_\omega \leftarrow J_\omega+\lambda_{JS} \Big [\sum_{(x_j,y_j)\in Q^k}\log(1-f_\omega(f_\phi(x_j))) + \sum_{(x_j,y_j)\in S^k}\log(f_\omega(f_\psi(x_j))) \Big ]$
        % }
    \EndFor
\end{algorithmic}
\end{algorithm}

%-------------------------------------------------------------------------
\subsection{From the Theory to TKNets}
\label{sec:theory-to-implement}
 % in Sec.~\ref{assumption_comparison}
We further demonstrate the connection between the theory and our proposed method TKNets 
% We derive that the loss in Eq.~\ref{eqn:loss} of our DPNets is the upper bound of the sum of two terms: the Kullback–Leibler divergence loss defined in Def.~\ref{def1} and Inter-Class Distance loss 
(the complete proof is deferred Section~\ref{appendix:proof_of_loss_kl} in Appendix):
\begin{theorem}

\label{eqn:inequality}
The optimization loss $J$ defined in Eq.~(\ref{eqn:loss}), is the approximation of the upper bound of KL term in Def.~\ref{def1} and an inter-class distance loss, which implies

{
\begin{align}
    &\underbrace{- \sum_{i=2}^{m}E_{y'\sim \mathcal D_{i}(y')} E_{ y\sim \mathcal 
    D_{i}^{\mathcal{K}^*}(y), 
         y\neq y'}  d(\mathcal{D}_{i}(z|y')||\mathcal{D}^{\mathcal{K}^*}_{i}(z|y))}_{\text{Inter-Class Distance loss}} \nonumber\\
   & + \underbrace{\sum_{i=2}^{m}\mathbb E_{y\sim \mathcal D_{i}(y)}  d_{}(\mathcal D_{i}(z|y)||\mathcal D^{\mathcal{K} ^*}_{i}(z|y))}_{\text{KL-Divergence loss in Eq.~(\ref{eqn:bound})}}
    \le \frac{J}{K-1}
\end{align}
}
\end{theorem}

\noindent\textbf{Discussion} Minimizing the above KL term, which is the distance between the conditional distributions of forecasted and real domains, corresponds to minimizing the KL term in Eq.~(\ref{eqn:bound}), thereby reducing the generalization bound. The first term can be regarded as a regularizer term that achieves maximization of the inter-class dissimilarity.
% The above rationale is consistent with the behavior of the prototypical network~\cite{proto}.
% (The rationale is similar to Linear Discriminant Analysis that minimizes the distances between samples of the same class but enlarges the distance between samples from different classes).

%-------------------------------------------------------------------------
\section{Experiments}\label{sec:exp}

% \begin{center}
% \adjustbox{max width=\textwidth}{%
% \begin{tabular}{lccccccc}
% % \label{summary_table}
% \toprule
% \textbf{Algorithm}        & \textbf{RMNIST}     & \textbf{EDGCircle}        & \textbf{EDGEvolCircle}    & \textbf{EDGSine}          & \textbf{EDGPortrait}      & \textbf{EDGForestCover}   & \textbf{Avg}              \\
% \midrule
% \bottomrule
% \end{tabular}}
% \end{center}

% In this section, we evaluate our method with 

\subsection{Experimental Setup}
\label{exps_setup}

We evaluated our algorithm on 6 datasets, 4 of which were collected in real-world scenarios (RMNIST \cite{wjd_59}, Portrait \cite{usgda,chen2021gradual}, Cover Type \cite{usgda}), and FMoW \cite{christie2018functional}).
%Furthermore, we also include both image-based and feature-based data sets to demonstrate the broad applicability of our approach in real-world scenarios. 
% The empirical results on these data sets showed the effectiveness of our algorithm in the EDG setting by leveraging the evolving patterns. 

% To simulate the evolving pattern on the joint probability distributions among different domains, we generate Rotated Plate data set and Evolving Circle data set by restricting the evolving manners only on $P(X)$ or $P(Y|X)$ respectively.
% For the synthetic data sets part, our designing motivation is that the difference between domains can be mathematically portrayed as a difference in the joint probability distribution $P(X, Y)$ of the data points in them. In further, since $P(X,Y)$ can be decomposed into $P(X)$ and $P(Y|X)$, we generate two data data set to simulate the scenario where only $P(X)$ or only $P(Y|X)$ evolves respectively.

%The protocal of each data set is shown as following: 
(1) \textbf{Evolving Circle (EvolCircle, Fig. \ref{fig:evolcircle})} consists of 30 evolving domains, where the instances are generated from 30 2D-Gaussian distributions with the same variances but different centers uniformly distributed on a half-circle. 
% The domains and decision boundaries are shown in Fig. \ref{fig:evolcircle}.
% The motivation of designing such a data set is to simulate the scenario where $P(X)$ evolves across domains while $P(Y|X)$ remains consistent (the half-circle boundary between positive and negative classes). 
% An example for such a DG task in real-world scenarios can be the image classification problem between tigers and cats across different visual domains (e.g., photo, painting and sketch), where the $P(X)$ varies across different domains due to the dissimilar picture styles. While the relationship between the features that allow us to identify the two animals, such as teeth and patterns, and their labels is consistent, which makes $P(Y|X)$ be consistent. \LJQ{The example is not good...}
(2) \textbf{Rotated Plate (RPlate, Fig. \ref{fig:rplate})} consists of 30 domains, where the instances of each domain are generated by the same Gaussian distribution but the decision boundary rotates from $0^\circ$ to $348^\circ$ with an interval of $12^\circ$. 
(3) \textbf{Rotated MNIST (RMNIST)} We randomly select only 2400 instances in the raw MNIST dataset and split them into 12 domains equally. Then we apply the rotations with degree of $\theta=\{0^\circ, 10^\circ, ..., 110^\circ\}$ on each domain respectively. The amount of samples in each domain is only 200, which makes this task more challenging.
% the performance  \emph{source domain } different situations, we generate different versions of the rotated MNIST by adjusting the number of domains and the rotation degree, which can be interpreted as the distance between domains. 
(4) \textbf{Portrait} data was originally proposed in \cite{portrait} and has been used as a benchmark dataset for studying evolving domain adaptation \cite{chen2021gradual, usgda} and other problems related to evolving domains~\cite{chen2021active, lei2021near, mancini2019adagraph, zhou2022online}. It aims to predict gender based on the photos of high school seniors across different decades. We divided the dataset into 12 domains by year. 
(5) \textbf{Cover Type} data set of geology aims to predict cover type (the predominant kind of tree cover) from 54 strictly cartographic variables. 
%The ground truth label of cover type for a given 30-square-meters cell was determined from US Forest Service. 
To generate evolving domains, we sort the samples by the ascending order of the distance to the water body, as proposed in \cite{usgda}. Then we equally divided the data set into 10 domains by distance. 
%Then we equally divide the distance range into 10 intervals and collect the samples within the same distance interval as a domain.
% A real-world image classification data set, which is comprised of photos of high school seniors across decades (\cite{portrait}). The goal is to predict gender base on one singe image. 
% The gap among different domains is mainly caused by the variations in hairstyle across different decades.  \WW{TODO splitting criteria}
% For example, the same gender in different generations may have different degrees of tendency to choose long hair.
(6) \textbf{FMoW} A large satellite image dataset with target detection and classification tasks \cite{christie2018functional}. We select 5 common classes to compose a classification task. The dataset is divided into 19 domains by time. 

\begin{figure*}[tbp]
		\centering 
% 		\subfloat[Domains]{\includegraphics[width=0.25\textwidth]{EXPS/Paper3/rainbow_Circle_notation.png}}
% 		\subfloat[Ground Truth]{\includegraphics[width=0.25\textwidth]{EXPS/Paper3/Label_Circle_Notation.png}}
% 		\subfloat[ERM]{\includegraphics[width=0.25\textwidth]{EXPS/Paper3/Prediction_Circle_ERM_Notation.png}}
% 		\subfloat[EDG]{\includegraphics[width=0.25\textwidth]{EXPS/Paper3/Prediction_Circle_EDG_Notation.png}}
		% \includegraphics[width=1.0\textwidth]{EXPS/Paper3/EvolCircle2.pdf}
            \includegraphics[width=0.88\textwidth]{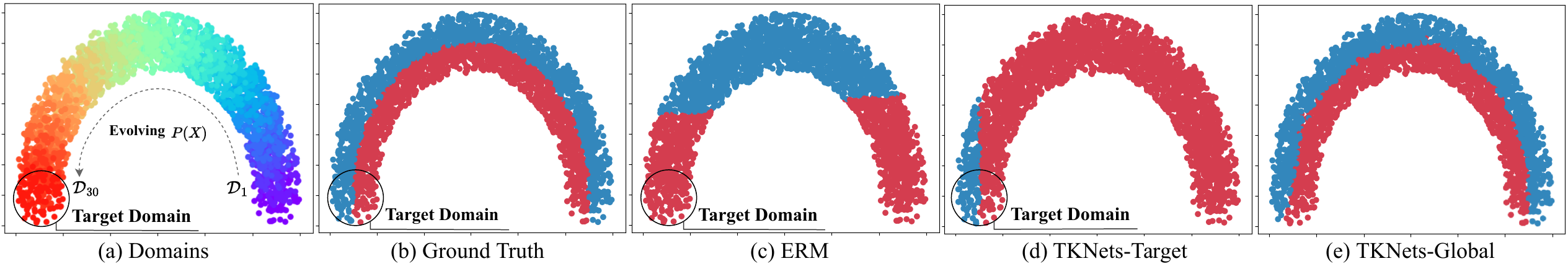}
\caption{Visualization of the EvolCircle dataset. (a) 30 domains indexed by different colors, where the left bottom one is the target domain.  (b) Positive and negative instances are denoted by red and blue dots respectively. (c) The decision boundaries learned by ERM. (d) Decision boundaries of the last model on all domains. (e) Decision boundaries of models in each domain.}
\label{fig:evolcircle}
\end{figure*}

\begin{figure*}[tbp]
		\centering 
		\includegraphics[width=0.88\textwidth]{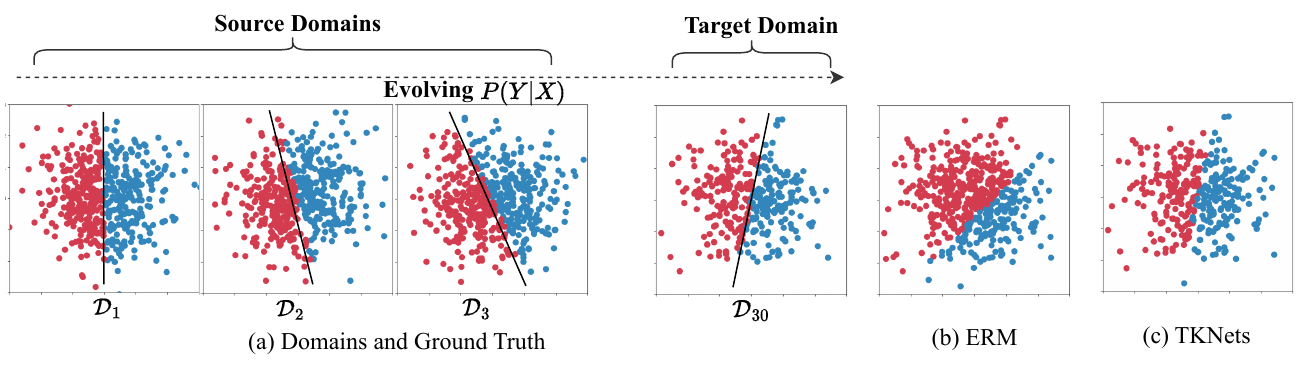}
\caption{Visualization of the RPlate data set. (a) True decision boundaries evolve over domains. (b) \& (c) The decision boundaries learned by ERM and TKNets on the target domain. }
% ERM fails to generalize to target domain while DPNets achieve a good classification boundary by capturing the evolving patterns.}
\label{fig:rplate}
\end{figure*}
% (6) \textbf{Drought} A large scale drought prediction task with only meteorological data \citep{drought}. We convert it into  a classification task with six levels of drought and divided it into 40 domains by time.
% (7) \textbf{OcularDisease} A medical dataset for ophthalmic disease identification based on images of the eyeball \citep{ocular}. The dataset is divided in to 12 domains by age.

% We use this data set to simulate the scenario where $P(X)$ remains unchanged while the $P(Y|X)$ evolves. \WW{TODO add examples?}

% \subsection{Implementation and Baselines}
We compared the proposed method with the following baselines: 
(1) \textbf{GroupDRO}~\cite{sagawa2019distributionally}; 
(2) \textbf{MLDG}~\cite{li2018learning}; (3) \textbf{MMD}~\cite{li2018domain}; 
(4) \textbf{SagNet}~\cite{nam2021reducing}; 
(5) \textbf{VREx}~\cite{krueger2021out}; 
(6) \textbf{SD}~\cite{pezeshki2020gradient}; 
(7) \textbf{IRM}~\cite{wjd_88}; 
(8) \textbf{Mixup}~\cite{yan2020improve}; 
(9) \textbf{CORAL}~\cite{sun2016deep}; 
(10) \textbf{MTL}~\cite{blanchard2021domain}; 
(11) \textbf{RSC}~\cite{huang2020self}; 
(12) \textbf{DIRL}~\cite{dirt}; 
(13) \textbf{ERM}~\cite{vapnik1998Statistical}; 
(14) \textbf{ERM-Near} Only using the most adjacent domains for training; 
(15) \textbf{ERM-W} ERM with weighted loss where the samples from domains closer to the target have higher weights;
(16) \textbf{GI}~\cite{gi};
(17) \textbf{LSSAE}~\cite{lssae};
(18) \textbf{CIDA}~\cite{cida};
(19) \textbf{EAML}~\cite{laed}.
All the baselines and experiments were implemented with DomainBed package \cite{domainbed} under the same setting, which guarantees extensive and sufficient comparisons. 
%Then we report the numerical results under optimal hyperparameters in Table \ref{tab:all_res}. 

   %% \textbf{EvolCircle} and \textbf{RPlate} use identity mapping as an embedding function $\phi$, embedding size 2. \textbf{RMNIST} dataset equips with MNIST-CNN backbone as $\phi$ from \cite{domainbed}, embedding size 128. \textbf{Portrait} use a Wide-ResNet from \cite{domainbed} as $\phi$ and embedding size 128. \textbf{Cover Type} use an MLP as backbone $\phi$, and embedding size 128. \textbf{FMoW} use a ResNet-50 as backbone $\phi$ from \cite{domainbed} and embedding size 2048. All $g$ is set to MLP, which are all predefined in DomainBed. Other details of hyperparameters and experimental setup are provided in Section \ref{appendix:implement detail} of the Appendix. 

%For the simple data sets, including RPlate and EvolCircle, we use single-layer linear mapping networks for both $f_\phi$ and $f_\psi$. As for the remaining three datasets, we provide the details of the mapping model in the appendix.

\begin{table*}[!t]
\caption{Comparison of accuracy (\%) among different methods.  }
\label{tab:all_res}
\begin{center}
\adjustbox{max width=0.77\textwidth}{%
\begin{tabular}{lccccccc}
\toprule
\textbf{Algorithm}         & \textbf{EvolCircle}  & \textbf{RPlate}  & \textbf{RMNIST} & \textbf{Portrait}  & \textbf{Cover Type} & \textbf{FMoW} & \textbf{Average}              \\
\midrule
ERM & 72.7 $\pm$ 1.1   & 63.9 $\pm$ 0.9   & 79.4 $\pm$ 0.0   & 95.8 $\pm$ 0.1   & 71.8 $\pm$ 0.2   & 54.6 $\pm$ 0.1  & 74.7\\
ERM-Near & 90.6 $\pm$ 1.4   & 94.1 $\pm$ 1.4   &  70.2 $\pm$ 1.3   & 80.9 $\pm$ 0.3   & 65.4 $\pm$ 0.2   & 44.2 $\pm$ 0.1  & 74.2\\
ERM-W & 80.4 $\pm$ 1.1   & 73.4 $\pm$ 0.5 & 82.0 $\pm$ 1.1   & 94.9 $\pm$ 0.3   & 71.5 $\pm$ 0.2   & 50.2 $\pm$ 0.1  & 75.4\\
GroupDRO          & 75.5 $\pm$ 1.0   & 70.0 $\pm$ 4.9   & 76.5 $\pm$ 0.2   & 94.8 $\pm$ 0.1   & 66.4 $\pm$ 0.5   & 57.3 $\pm$ 0.1   & 73.4                                  \\
MLDG                  & 91.5 $\pm$ 2.0   & 66.9 $\pm$ 1.8   & 75.0 $\pm$ 0.3   & 66.2 $\pm$ 1.7   & 68.4 $\pm$ 0.7   & 43.8 $\pm$ 0.0   & 68.6                                  \\
MMD                    & 86.7 $\pm$ 5.7   & 59.9 $\pm$ 1.4   & 35.4 $\pm$ 0.0   & 95.4 $\pm$ 0.1   & 69.8 $\pm$ 0.4   & 60.0 $\pm$ 0.0   & 67.8                                 \\
SagNet                  & 78.7 $\pm$ 3.2   & 63.8 $\pm$ 2.9   & 79.4 $\pm$ 0.1   & 95.3 $\pm$ 0.1   & 65.3 $\pm$ 2.2   & 56.2 $\pm$ 0.1  & 73.1                                 \\
VREx                     & 82.9 $\pm$ 6.6   & 61.1 $\pm$ 2.6   & 79.4 $\pm$ 0.1   & 94.3 $\pm$ 0.2   & 66.0 $\pm$ 0.9   & 61.2 $\pm$ 0.0   & 73.3                                 \\
SD                          & 81.7 $\pm$ 4.3   & 65.3 $\pm$ 1.4   & 78.8 $\pm$ 0.1   & 95.1 $\pm$ 0.2   & 69.1 $\pm$ 0.9   & 55.2 $\pm$ 0.0   & 74.2                                  \\
IRM                       & 86.2 $\pm$ 3.0   & 67.2 $\pm$ 2.1   & 47.5 $\pm$ 0.4   & 94.4 $\pm$ 0.3   & 66.0 $\pm$ 1.0   & 58.8 $\pm$ 0.0  & 70.0                           \\
Mixup                   & 91.5 $\pm$ 2.6   & 66.8 $\pm$ 1.8   & 81.3 $\pm$ 0.2   & 96.4 $\pm$ 0.2   & 69.7 $\pm$ 0.6   & 59.5 $\pm$ 0.0   & 77.5                                 \\
CORAL                 & 86.8 $\pm$ 5.1   & 61.9 $\pm$ 1.4   & 78.4 $\pm$ 0.1   & 95.1 $\pm$ 0.1   & 68.1 $\pm$ 1.3   & 56.1 $\pm$ 0.0   & 74.4                                  \\
MTL                      & 77.7 $\pm$ 2.4   & 66.0 $\pm$ 1.2   & 77.2 $\pm$ 0.0   & 95.4 $\pm$ 0.1   & 69.2 $\pm$ 0.9   & 51.7 $\pm$ 0.0   & 72.9                                 \\
RSC                       & 91.5 $\pm$ 2.1   & 67.9 $\pm$ 4.2   & 74.7 $\pm$ 0.1   & 95.5 $\pm$ 0.1   & 69.4 $\pm$ 0.3   & 55.7 $\pm$ 0.1   & 75.8                                 \\
DIRL & 53.3 $\pm$ 0.2   & 56.3 $\pm$ 0.4   & 76.3 $\pm$ 0.3   & 93.2 $\pm$ 0.2   & 61.2 $\pm$ 0.3 & 43.4 $\pm$ 0.3  & 64.0\\

% Prototypical& 93.6 $\pm$ 0.5   & 66.3 $\pm$ 0.4   & 85.2 $\pm$ 0.4   & 96.2 $\pm$ 0.3   & 66.5 $\pm$ 0.4  & 53.3 $\pm$ 0.2 & 76.9\\
CIDA & 68.0 $\pm$ 2.8   & 91.6 $\pm$ 2.5   & 86.5 $\pm$ 1.2   & 94.5 $\pm$ 1.2   & 70.2 $\pm$ 0.3   & 63.4 $\pm$ 0.2  & 70.9\\
EAML & 90.1 $\pm$ 1.3   & 91.2 $\pm$ 0.3 & 83.1 $\pm$ 0.8   & 95.2 $\pm$ 0.1   & 69.3 $\pm$ 0.1   & 61.1 $\pm$ 0.2  & 81.7\\

LSSAE & 91.5 $\pm$ 2.3  & 91.7 $\pm$ 1.3 & 86.0 $\pm$ 0.2   & 96.0 $\pm$ 0.2   & 71.3 $\pm$ 0.5   & 60.5 $\pm$ 0.1  & 82.8\\
GI & 92.1 $\pm$ 1.7   & 94.3 $\pm$ 0.4 & 86.3 $\pm$ 0.2   & 96.3 $\pm$ 0.1  & 71.7 $\pm$ 0.3   & 64.5 $\pm$ 0.1  & 83.3\\

% DPNets (Ours)& \textbf{94.2 $\pm$ 0.9}   & \textbf{95.0 $\pm$ 0.5}   & \textbf{87.5 $\pm$ 0.1}   & \textbf{96.4 $\pm$ 0.0}   & \textbf{72.5 $\pm$ 1.0} & \textbf{66.8 $\pm$ 0.1}  & \rebuttal{\textbf{85.4}}\\
TKNets (Ours)& \textbf{94.2 $\pm$ 0.9}   & \textbf{95.0 $\pm$ 0.5}   & \textbf{87.5 $\pm$ 0.1}   & \textbf{97.2 $\pm$ 0.0}   & \textbf{73.8 $\pm$ 1.0} & \textbf{66.8 $\pm$ 0.1 } & \textbf{85.7}\\
\bottomrule
\end{tabular}}
\end{center}
\end{table*}

\subsection{Results and Analysis}

\textbf{Overall Evaluation} The performances of our proposed method and baselines are reported in Table \ref{tab:all_res}. It can be observed that TKNets consistently outperforms extensive DG baselines over all the data sets and achieves $85.7\%$ on average which is significantly higher than the other algorithms ($\approx 8\% - 20\%$). The results indicate that existing DG methods cannot deal with evolving domain shifts well while TKNets can properly capture the evolving patterns in the environments. 
% Comparison with Prototypical Network can be considered as an ablation study to evaluate the impact of leveraging evolving information, the average boost of $8.5\%$ proves its effectiveness. 
ERM-Near is a simple solution for TDG as the most recent domain shift least with respect to the target.
It performs well on synthetic data as the samples of one domain are enough for training a classifier while the low sample utilization makes it hard to handle real datasets. Section~\ref{appendix:domain_info} in the appendix presents an experiment on ERM with domain index, which shows that directly incorporating it does not improve learning in the TDG problem.
We also implement CIDA \cite{cida}, as a competitive TDA method. 
% under DomainBed \cite{domainbed} framework. 
With the fair random hyper-parameter search strategy, the accuracy of TKNets is $14.5\%$ higher on average. The results illustrate the superiority of TKNets.
\begin{table*}[ht]
\caption{Comparison of accuracy (\%) of different methods on RMNIST data set with different number of domains.}
\label{tab:domain_num_RMNIST_table}
\begin{center}
\adjustbox{max width=0.74\textwidth}{%
\begin{tabular}{lccccccccc}
\toprule
\multicolumn{1}{c}{\# Domains} & 7                       & 9                       & 11                      & 13                      & 15                      & 17 &19                      \\
\midrule
Mixup                             & \textbf{83.8 $\pm$ 0.6}          & \textbf{83.3 $\pm$ 0.3}          & 80.0 $\pm$ 0.3          & 78.3 $\pm$ 0.3          & 80.0 $\pm$ 0.3          & 77.3 $\pm$ 0.3    & 72.5 $\pm$ 0.3        \\
IRM                             & 35.6 $\pm$ 0.3          & 46.5 $\pm$ 0.3          & 40.4 $\pm$ 0.4          & 49.6 $\pm$ 0.3          & 46.0 $\pm$ 0.1          & 46.9 $\pm$ 0.3   & 41.3 $\pm$ 0.1         \\
MLDG                             &82.7 $\pm$ 0.3          & 80.0 $\pm$ 0.6          & 79.0 $\pm$ 0.1          & 74.2 $\pm$ 0.1          & 77.9 $\pm$ 0.3          & 71.7 $\pm$ 0.1  & 68.8 $\pm$ 0.3          \\
ERM                            & 81.3 $\pm$ 0.1          & 79.7 $\pm$ 0.2          & 79.7 $\pm$ 0.3          & 75.6 $\pm$ 0.3          & 77.8 $\pm$ 0.3          & 69.1 $\pm$ 0.3   & 74.4 $\pm$ 0.1         \\
TKNets (Ours)                             & 81.1 $\pm$ 0.1 & 82.8 $\pm$ 0.3 & \textbf{88.1 $\pm$ 0.3} & \textbf{87.3 $\pm$ 0.5} & \textbf{86.6 $\pm$ 0.1} & \textbf{85.6 $\pm$ 0.3} & \textbf{86.3 $\pm$ 0.3}\\
\bottomrule
\end{tabular}
}
\end{center}
\end{table*}

\begin{table*}[htbp]
\caption{Comparison of accuracy (\%) of different methods on RMNIST data set with different distance between domains.  }
\label{tab:domain_dis_RMNIST_table}
\begin{center}
\adjustbox{max width=0.7\textwidth}{%
\begin{tabular}{lcccccc}
\toprule
\multicolumn{1}{c}{Domain Distance}& $3^\circ$                       & $5^\circ$                       & $7^\circ$                       & $10^\circ$                      & $15^\circ$                      & $20^\circ$                      \\
\midrule
Mixup                               & 92.5 $\pm$ 0.1          & 91.9 $\pm$ 0.3          & 88.4 $\pm$ 0.3          & 81.3 $\pm$ 0.2          & 73.1 $\pm$ 0.1          & 59.4 $\pm$ 0.3          \\
IRM                                     & 69.4 $\pm$ 0.2          & 63.4 $\pm$ 0.1          & 49.7 $\pm$ 0.0          & 47.5 $\pm$ 0.4          & 35.6 $\pm$ 0.3          & 31.3 $\pm$ 0.3          \\
MLDG                                & 90.9 $\pm$ 0.1          & 87.5 $\pm$ 0.2          & 85.0 $\pm$ 0.2          & 75.0 $\pm$ 0.3          & 71.9 $\pm$ 0.1          & 56.9 $\pm$ 0.3          \\
ERM                                    & 92.2 $\pm$ 0.1          & 88.8 $\pm$ 0.0          & 82.8 $\pm$ 0.1          & 79.4 $\pm$ 0.0          & 66.3 $\pm$ 0.1          & 53.1 $\pm$ 0.3          \\
TKNets (Ours)                             & \textbf{94.7 $\pm$ 0.3} & \textbf{93.1 $\pm$ 0.3} & \textbf{90.8 $\pm$ 0.3} & \textbf{88.3 $\pm$ 0.1} & \textbf{86.4 $\pm$ 0.3} & \textbf{84.1 $\pm$ 0.3}\\
\bottomrule
\end{tabular}
}
\end{center}
\end{table*}

\begin{table}[t]
    \caption{Performance on multiple target domains.}
    \label{tab:multiple_step}
    \begin{center}
    
    \adjustbox{max width=0.46\textwidth}{%
    \begin{tabular}{lcccc}
    \toprule
    \textbf{Domain}         & \textbf{T+1}  & \textbf{T+2} & \textbf{T+3} & \textbf{T+4}             \\
    \midrule
        ERM & 79.4 $\pm$ 0.0 & 65.0 $\pm$ 1.1 & 50.7 $\pm$ 1.2 & 41.0 $\pm$ 1.7 \\ 
        GI & 86.3 $\pm$ 0.2 & 65.9 $\pm$ 3.1 & 48.5 $\pm$ 1.2 & 41.1 $\pm$ 1.1 \\ 
        EAML & 83.1 $\pm$ 0.8 & 67.5 $\pm$ 1.5 & 55.6 $\pm$ 1.7 & 44.1 $\pm$ 1.6 \\ 
        TKNets & \textbf{87.5 $\pm$ 0.1} & \textbf{83.1 $\pm$ 0.8} & \textbf{72.9 $\pm$ 2.4} & \textbf{72.2 $\pm$ 1.1} \\ 
    \bottomrule
    \end{tabular}}
    \end{center}
\end{table}

\noindent\textbf{Evolving $P(X)$ (EvolCircle).} 
% On this data set, we compared the performance of DPNets and traditional DG methods on evolving tasks where only $P(X)$ varies.
%As shown in Fig.~\ref{fig:evolcircle}, we provide the 
The decision boundaries on the unseen target domain $\mathcal{D}_{30}$ learned by ERM and TKNets are shown in Fig.~\ref{fig:evolcircle}(c) and Fig.~\ref{fig:evolcircle}(d) respectively. 
% By comparing the predictions on target domain $\mathcal D_{30}$, 
We can observe that TKNet fits the ground truth significantly better than ERM. This indicates that our approach can capture the evolving pattern of $P(X)$ according to source domains and then learn a better classifier for the target domain. 
Furthermore, we can observe that the decision boundary learned by ERM achieves better performance on the observed source domains. This is because it focuses on improving generalization ability on all source domains, which leads the poor performance on the outer target domain $\mathcal D_{30}$. Instead, TKNets can ``foresee" the centroids for the target domain, which guarantees a good generalization performance, though it may not perform well on source domains. 

% if we compare the decision boundaries of ERM and DPNets, we can notice our approach is worse than ERM if we calculate the average accuracy over all the domains. This is because the traditional DG methods focus on improving generalization ability according to the global distributions on all source domains, which leads the poor performance on the ``outer" target domain $\mathcal D_{30}$. On the contrary, our DPNets generate prototypes for the ``outer" target domain, which guarantee a good generalization performance.

\noindent\textbf{Evolving $P(Y|X)$ or $P(X|Y)$ (RPlate).} 
%In this part, we compare the generalization ability of DPNets and traditional DG methods on the evolving $P(Y|X)$ tasks. 
By visualizing the datasets, we can observe that the predicted boundary of TKNets better approximates the ground truth, compared with the result of ERM. This indicates that our approach can also capture the $P(Y|X)$ evolving pattern. Existing DG methods perform poorly on this dataset because the ground truth labeling function varies. Under the evolving labeling functions, even the same instance can have different labels in different domains. Thus, there does not exist a single model that can perform well across all the domains. For this situation, learning a model specifically for one domain instead of all domains can be a possible solution. TKNets can capture the evolving pattern and produce a  model specifically for the target domain.

\noindent\textbf{When to apply TKNets?} Existing DG methods assume that the distances between observed and unseen domains are not very large. However, the dissimilarity between domains is a crucial factor that can fundamentally influence the possibility and performance of generalization. To investigate the impact of variances of the environment, we create a series of variations on the raw RMNST data by jointly varying the number of domains (Table \ref{tab:domain_num_RMNIST_table}) and the degree interval (Table \ref{tab:domain_dis_RMNIST_table}) between two consecutive domains. 
% The performance improvement of TKNets over the baseline ERM ($\text{Acc}_{\text{TKNets}} - \text{Acc}_{\text{ERM}}$) is shown in Fig. \ref{3d}. camera-ready comment
On the one hand, the greater number of domains and larger distance between them lead to more significant differences across domains. 
This makes traditional DG methods harder to train one model from all domains, but instead more domains benefit our TKNets to learn the evolving pattern to achieve better performance. On the other hand, we observe TKNets significantly outperform other baselines when the number of domains and the distance between domains increase. 
% Under this situation, even though the domains are similar, there is less useful directional information. However, the models performing well on a single domain will also gain good performance on the others. 

% In Table \ref{tab:domain_num_RMNIST_table} and Table \ref{tab:domain_dis_RMNIST_table}, we respectively analyzed the effect of \emph{the domain distance} and \emph{the number of domains} on the generalization ability of different models. With the increasing complexity of source domains, the DPNet benefits from the evolving information and the performance gap between our method and baselines increases. 
In Table \ref{tab:domain_num_RMNIST_table}, the performance of traditional DG methods fluctuate when the number of domains increases. TKNets' performance continuously improves when the domain number increases since it easily learns the evolving manners from more domains. Please refer to Section \ref{appendix:inter_extra} for more discussion about this.
From Table \ref{tab:domain_dis_RMNIST_table}, when the domain distance increases, the performance of DG methods decreases severely while the performance of TKNets drops slightly. 

% Above all, the experimental results imply it is hard for traditional DG methods to solve TDG problems when the number of domains and distance between domains increases, but TKNets can perform well in such a scenario. 

% After training on 11 source domains, we test the performance on the unseen 4 target domains. 
\noindent\textbf{Domain Generalization across Temporal Domains.}
There are many scenarios in the real world where it is necessary for us to have the model applied to the next several domains instead of one. In this section, to test the performance of the algorithm on more distant domains, we extend the RMNIST \cite{wjd_59} to 15 domains, with the last four domains as target domains. From the Table~\ref{tab:multiple_step}, as expected, the model performance drops heavily as the index increases due to the domain shifting. However, the accuracy of our algorithm decreases very slowly compared to other algorithms. Even on $T+4$ domains, the accuracy still remains at 72.2\%. While the accuracy of other algorithms has dropped to $41.0\% \sim 44.1\%$. The result demonstrates the generalization performance of our algorithm on multiple domains.

% \textbf{The Impact of the Choice of the Mapping Function $g$.} The only component beside the backbone is {\color{black}mapping function} $g$. To analyze the importance of $g$ and its contribution to DPNets and the overall performance, we replace it with Attention~\cite{vaswani2017attention} and LSTM~\cite{hochreiter1997long} in Table~\ref{Table:g}. The result shows that Attention-based and LSTM-based achieve similar accuracy with MLP-based DPNets. 
% \vspace{-0.cm}
% \begin{table}[ht!]
% \caption{Experiment on the choice of mapping function $g$}
% \label{Table:g}
% \setlength{\tabcolsep}{3.5pt}
% \vskip -0.in
% \vspace{-0.cm}
% \begin{center}
% \begin{sc}
% \vspace{-0.5cm}
% \scalebox{0.9}{
% \begin{tabular}{cccc}
% \toprule
% \textbf{Algorithm} & P & RM  \\
% \midrule
% MLDG + GAN & 92.4 & 85.4\\
% MixUp & 92.3 & 85.7\\
% DAML & 92.7 & 86.4\\
% Ours  & 94.6 & 88.4\\

% \textbf{Algorithm}& DPNets & +Attention & +LSTM  \\
% \midrule

% RMNIST & 87.5 $\pm$ 0.1 &  87.2 $\pm$ 0.2 & \textbf{87.6 $\pm$ 0.1} \\
% Portrait & 96.4 $\pm$ 0.0 &  \textbf{96.5 $\pm$ 0.0} & 96.4 $\pm$ 0.1 \\
% % \midrule
% \bottomrule
% \end{tabular}}
% \vspace{-0.5cm}
% \end{sc}
% \end{center}
% \end{table}

% \vspace{-0.5cm}
\section{Conclusions}
In this paper, we study the problem of domain generalization in an evolving environment and propose temporal domain generalization (TDG)  as a general framework to address it. Our theoretical analysis highlights the role of learning a Koopman operator to capture the evolving pattern over domains. Motivated by our theory, we propose temporal Koopman networks (TKNets), a simple and efficient algorithm for TDG. Experiments on both synthetic and real-world datasets validate the effectiveness of our method. 

\section*{Acknowledgements}
We appreciate constructive feedback from anonymous reviewers and meta-reviewers. This work is supported by the Natural Sciences and Engineering Research Council of Canada (NSERC), Discovery Grants program.

%%%%%%%%% REFERENCES
\small{\bibliography{example_paper}}

\begin{thebibliography}{64}
\providecommand{\natexlab}[1]{#1}

\bibitem[{Arjovsky et~al.(2019)Arjovsky, Bottou, Gulrajani, and Lopez-Paz}]{wjd_88}
Arjovsky, M.; Bottou, L.; Gulrajani, I.; and Lopez-Paz, D. 2019.
\newblock Invariant risk minimization.
\newblock \emph{arXiv preprint arXiv:1907.02893}.

\bibitem[{Azencot et~al.(2020)Azencot, Erichson, Lin, and Mahoney}]{azencot2020forecasting}
Azencot, O.; Erichson, N.~B.; Lin, V.; and Mahoney, M. 2020.
\newblock Forecasting sequential data using consistent Koopman autoencoders.
\newblock In \emph{International Conference on Machine Learning}, 475--485. PMLR.

\bibitem[{Bai, Ling, and Zhao(2022)}]{bai2022temporal}
Bai, G.; Ling, C.; and Zhao, L. 2022.
\newblock Temporal Domain Generalization with Drift-Aware Dynamic Neural Networks.
\newblock \emph{arXiv preprint arXiv:2205.10664}.

\bibitem[{Balaji, Sankaranarayanan, and Chellappa(2018)}]{wjd_13}
Balaji, Y.; Sankaranarayanan, S.; and Chellappa, R. 2018.
\newblock Metareg: Towards domain generalization using meta-regularization.
\newblock \emph{Advances in Neural Information Processing Systems}, 31: 998--1008.

\bibitem[{Ben-David et~al.(2010)Ben-David, Blitzer, Crammer, Kulesza, Pereira, and Vaughan}]{david}
Ben-David, S.; Blitzer, J.; Crammer, K.; Kulesza, A.; Pereira, F.; and Vaughan, J.~W. 2010.
\newblock A theory of learning from different domains.
\newblock \emph{Machine learning}, 79(1): 151--175.

\bibitem[{Blanchard et~al.(2021)Blanchard, Deshmukh, Dogan, Lee, and Scott}]{blanchard2021domain}
Blanchard, G.; Deshmukh, A.~A.; Dogan, {\"U}.; Lee, G.; and Scott, C. 2021.
\newblock Domain Generalization by Marginal Transfer Learning.
\newblock \emph{J. Mach. Learn. Res.}, 22: 2--1.

\bibitem[{Blanchard, Lee, and Scott(2011)}]{wjd_24}
Blanchard, G.; Lee, G.; and Scott, C. 2011.
\newblock Generalizing from several related classification tasks to a new unlabeled sample.
\newblock \emph{Advances in neural information processing systems}, 24: 2178--2186.

\bibitem[{Carlucci et~al.(2019)Carlucci, D'Innocente, Bucci, Caputo, and Tommasi}]{wjd_114}
Carlucci, F.~M.; D'Innocente, A.; Bucci, S.; Caputo, B.; and Tommasi, T. 2019.
\newblock Domain generalization by solving jigsaw puzzles.
\newblock In \emph{Proceedings of the IEEE/CVF Conference on Computer Vision and Pattern Recognition}, 2229--2238.

\bibitem[{Chen and Chao(2021)}]{chen2021gradual}
Chen, H.-Y.; and Chao, W.-L. 2021.
\newblock Gradual Domain Adaptation without Indexed Intermediate Domains.
\newblock \emph{Advances in Neural Information Processing Systems}, 34.

\bibitem[{Chen et~al.(2021)Chen, Luo, Ma, and Zhang}]{chen2021active}
Chen, Y.; Luo, H.; Ma, T.; and Zhang, C. 2021.
\newblock Active online learning with hidden shifting domains.
\newblock In \emph{International Conference on Artificial Intelligence and Statistics}, 2053--2061. PMLR.

\bibitem[{Christie et~al.(2018)Christie, Fendley, Wilson, and Mukherjee}]{christie2018functional}
Christie, G.; Fendley, N.; Wilson, J.; and Mukherjee, R. 2018.
\newblock Functional map of the world.
\newblock In \emph{Proceedings of the IEEE Conference on Computer Vision and Pattern Recognition}, 6172--6180.

\bibitem[{Fuglede and Topsoe(2004)}]{jsd}
Fuglede, B.; and Topsoe, F. 2004.
\newblock Jensen-Shannon divergence and Hilbert space embedding.
\newblock In \emph{International Symposium onInformation Theory, 2004. ISIT 2004. Proceedings.}, 31. IEEE.

\bibitem[{Ganin and Lempitsky(2015)}]{wjd_77}
Ganin, Y.; and Lempitsky, V. 2015.
\newblock Unsupervised domain adaptation by backpropagation.
\newblock In \emph{International conference on machine learning}, 1180--1189. PMLR.

\bibitem[{Ghifary et~al.(2015)Ghifary, Kleijn, Zhang, and Balduzzi}]{wjd_59}
Ghifary, M.; Kleijn, W.~B.; Zhang, M.; and Balduzzi, D. 2015.
\newblock Domain generalization for object recognition with multi-task autoencoders.
\newblock In \emph{Proceedings of the IEEE international conference on computer vision}, 2551--2559.

\bibitem[{Ginosar et~al.(2015)Ginosar, Rakelly, Sachs, Yin, and Efros}]{portrait}
Ginosar, S.; Rakelly, K.; Sachs, S.; Yin, B.; and Efros, A.~A. 2015.
\newblock A Century of Portraits: {A} Visual Historical Record of American High School Yearbooks.
\newblock \emph{CoRR}, abs/1511.02575.

\bibitem[{Gulrajani and Lopez{-}Paz(2020)}]{domainbed}
Gulrajani, I.; and Lopez{-}Paz, D. 2020.
\newblock In Search of Lost Domain Generalization.
\newblock \emph{CoRR}, abs/2007.01434.

\bibitem[{Huang et~al.(2020{\natexlab{a}})Huang, Wang, Xing, and Huang}]{wjd_112}
Huang, Z.; Wang, H.; Xing, E.~P.; and Huang, D. 2020{\natexlab{a}}.
\newblock Self-challenging improves cross-domain generalization.
\newblock In \emph{Computer Vision--ECCV 2020: 16th European Conference, Glasgow, UK, August 23--28, 2020, Proceedings, Part II 16}, 124--140. Springer.

\bibitem[{Huang et~al.(2020{\natexlab{b}})Huang, Wang, Xing, and Huang}]{huang2020self}
Huang, Z.; Wang, H.; Xing, E.~P.; and Huang, D. 2020{\natexlab{b}}.
\newblock Self-challenging improves cross-domain generalization.
\newblock In \emph{Computer Vision--ECCV 2020: 16th European Conference, Glasgow, UK, August 23--28, 2020, Proceedings, Part II 16}, 124--140. Springer.

\bibitem[{Ilse et~al.(2020)Ilse, Tomczak, Louizos, and Welling}]{wjd_101}
Ilse, M.; Tomczak, J.~M.; Louizos, C.; and Welling, M. 2020.
\newblock Diva: Domain invariant variational autoencoders.
\newblock In \emph{Medical Imaging with Deep Learning}, 322--348. PMLR.

\bibitem[{Kim et~al.(2020)Kim, Yoo, Park, and Kim}]{cudaal}
Kim, J.; Yoo, S.; Park, G.; and Kim, J. 2020.
\newblock Continual Unsupervised Domain Adaptation with Adversarial Learning.
\newblock \emph{CoRR}, abs/2010.09236.

\bibitem[{Koopman(1931)}]{koopman1931hamiltonian}
Koopman, B.~O. 1931.
\newblock Hamiltonian systems and transformation in Hilbert space.
\newblock \emph{Proceedings of the National Academy of Sciences}, 17(5): 315--318.

\bibitem[{Krueger et~al.(2021)Krueger, Caballero, Jacobsen, Zhang, Binas, Zhang, Le~Priol, and Courville}]{krueger2021out}
Krueger, D.; Caballero, E.; Jacobsen, J.-H.; Zhang, A.; Binas, J.; Zhang, D.; Le~Priol, R.; and Courville, A. 2021.
\newblock Out-of-distribution generalization via risk extrapolation (rex).
\newblock In \emph{International Conference on Machine Learning}, 5815--5826. PMLR.

\bibitem[{Kumar, Ma, and Liang(2020)}]{usgda}
Kumar, A.; Ma, T.; and Liang, P. 2020.
\newblock Understanding self-training for gradual domain adaptation.
\newblock In \emph{International Conference on Machine Learning}, 5468--5479. PMLR.

\bibitem[{Lei, Hu, and Lee(2021)}]{lei2021near}
Lei, Q.; Hu, W.; and Lee, J. 2021.
\newblock Near-optimal linear regression under distribution shift.
\newblock In \emph{International Conference on Machine Learning}, 6164--6174. PMLR.

\bibitem[{Li et~al.(2021)Li, Wang, Zhang, Li, Keutzer, Darrell, and Zhao}]{li2021learning}
Li, B.; Wang, Y.; Zhang, S.; Li, D.; Keutzer, K.; Darrell, T.; and Zhao, H. 2021.
\newblock Learning invariant representations and risks for semi-supervised domain adaptation.
\newblock In \emph{Proceedings of the IEEE/CVF Conference on Computer Vision and Pattern Recognition}, 1104--1113.

\bibitem[{Li et~al.(2018{\natexlab{a}})Li, Yang, Song, and Hospedales}]{li2018learning}
Li, D.; Yang, Y.; Song, Y.-Z.; and Hospedales, T.~M. 2018{\natexlab{a}}.
\newblock Learning to generalize: Meta-learning for domain generalization.
\newblock In \emph{Thirty-Second AAAI Conference on Artificial Intelligence}.

\bibitem[{Li et~al.(2018{\natexlab{b}})Li, Pan, Wang, and Kot}]{li2018domain}
Li, H.; Pan, S.~J.; Wang, S.; and Kot, A.~C. 2018{\natexlab{b}}.
\newblock Domain generalization with adversarial feature learning.
\newblock In \emph{Proceedings of the IEEE Conference on Computer Vision and Pattern Recognition}, 5400--5409.

\bibitem[{Li et~al.(2020)Li, He, Wu, Katabi, and Torralba}]{li2020learning}
Li, Y.; He, H.; Wu, J.; Katabi, D.; and Torralba, A. 2020.
\newblock Learning Compositional Koopman Operators for Model-Based Control.
\newblock In \emph{International Conference on Learning Representations}.

\bibitem[{Li et~al.(2019)Li, Yang, Zhou, and Hospedales}]{wjd_14}
Li, Y.; Yang, Y.; Zhou, W.; and Hospedales, T. 2019.
\newblock Feature-critic networks for heterogeneous domain generalization.
\newblock In \emph{International Conference on Machine Learning}, 3915--3924. PMLR.

\bibitem[{Liu, Zhang, and Lu(2020)}]{liu2020heterogeneous}
Liu, F.; Zhang, G.; and Lu, J. 2020.
\newblock Heterogeneous domain adaptation: An unsupervised approach.
\newblock \emph{IEEE transactions on neural networks and learning systems}, 31(12): 5588--5602.

\bibitem[{Liu et~al.(2020)Liu, Long, Wang, and Wang}]{laed}
Liu, H.; Long, M.; Wang, J.; and Wang, Y. 2020.
\newblock Learning to Adapt to Evolving Domains.
\newblock In Larochelle, H.; Ranzato, M.; Hadsell, R.; Balcan, M.~F.; and Lin, H., eds., \emph{Advances in Neural Information Processing Systems}, volume~33, 22338--22348. Curran Associates, Inc.

\bibitem[{Long et~al.(2017)Long, Cao, Wang, and Jordan}]{long2017conditional}
Long, M.; Cao, Z.; Wang, J.; and Jordan, M.~I. 2017.
\newblock Conditional adversarial domain adaptation.
\newblock \emph{arXiv preprint arXiv:1705.10667}.

\bibitem[{Mancini et~al.(2018)Mancini, Bulo, Caputo, and Ricci}]{wjd_104}
Mancini, M.; Bulo, S.~R.; Caputo, B.; and Ricci, E. 2018.
\newblock Best sources forward: domain generalization through source-specific nets.
\newblock In \emph{2018 25th IEEE international conference on image processing (ICIP)}, 1353--1357. IEEE.

\bibitem[{Mancini et~al.(2019)Mancini, Bulo, Caputo, and Ricci}]{mancini2019adagraph}
Mancini, M.; Bulo, S.~R.; Caputo, B.; and Ricci, E. 2019.
\newblock Adagraph: Unifying predictive and continuous domain adaptation through graphs.
\newblock In \emph{Proceedings of the IEEE/CVF Conference on Computer Vision and Pattern Recognition}, 6568--6577.

\bibitem[{Muandet, Balduzzi, and Sch{\"o}lkopf(2013)}]{multi2}
Muandet, K.; Balduzzi, D.; and Sch{\"o}lkopf, B. 2013.
\newblock Domain generalization via invariant feature representation.
\newblock In \emph{International Conference on Machine Learning}, 10--18. PMLR.

\bibitem[{Nam et~al.(2021)Nam, Lee, Park, Yoon, and Yoo}]{nam2021reducing}
Nam, H.; Lee, H.; Park, J.; Yoon, W.; and Yoo, D. 2021.
\newblock Reducing Domain Gap by Reducing Style Bias.
\newblock In \emph{Proceedings of the IEEE/CVF Conference on Computer Vision and Pattern Recognition}, 8690--8699.

\bibitem[{Nasery et~al.(2021)Nasery, Thakur, Piratla, De, and Sarawagi}]{gi}
Nasery, A.; Thakur, S.; Piratla, V.; De, A.; and Sarawagi, S. 2021.
\newblock Training for the Future: A Simple Gradient Interpolation Loss to Generalize Along Time.
\newblock \emph{Advances in Neural Information Processing Systems}, 34: 19198--19209.

\bibitem[{Nguyen et~al.(2021)Nguyen, Tran, Gal, and Baydin}]{dirt}
Nguyen, A.~T.; Tran, T.; Gal, Y.; and Baydin, A.~G. 2021.
\newblock Domain Invariant Representation Learning with Domain Density Transformations.
\newblock In \emph{Thirty-Fifth Conference on Neural Information Processing Systems}.

\bibitem[{Pezeshki et~al.(2020)Pezeshki, Kaba, Bengio, Courville, Precup, and Lajoie}]{pezeshki2020gradient}
Pezeshki, M.; Kaba, S.-O.; Bengio, Y.; Courville, A.; Precup, D.; and Lajoie, G. 2020.
\newblock Gradient starvation: A learning proclivity in neural networks.
\newblock \emph{arXiv preprint arXiv:2011.09468}.

\bibitem[{Qin, Wang, and Li(2022)}]{lssae}
Qin, T.; Wang, S.; and Li, H. 2022.
\newblock Generalizing to Evolving Domains with Latent Structure-Aware Sequential Autoencoder.
\newblock In \emph{ICML}.

\bibitem[{Rahimian and Mehrotra(2019)}]{rahimian2019distributionally}
Rahimian, H.; and Mehrotra, S. 2019.
\newblock Distributionally robust optimization: A review.
\newblock \emph{arXiv preprint arXiv:1908.05659}.

\bibitem[{Rahman et~al.(2019)Rahman, Fookes, Baktashmotlagh, and Sridharan}]{wjd_38}
Rahman, M.~M.; Fookes, C.; Baktashmotlagh, M.; and Sridharan, S. 2019.
\newblock Multi-component image translation for deep domain generalization.
\newblock In \emph{2019 IEEE Winter Conference on Applications of Computer Vision (WACV)}, 579--588. IEEE.

\bibitem[{Ryu et~al.(2019)Ryu, Kwon, Yang, and Lim}]{wjd_115}
Ryu, J.; Kwon, G.; Yang, M.-H.; and Lim, J. 2019.
\newblock Generalized convolutional forest networks for domain generalization and visual recognition.
\newblock In \emph{International Conference on Learning Representations}.

\bibitem[{Sagawa et~al.(2019)Sagawa, Koh, Hashimoto, and Liang}]{sagawa2019distributionally}
Sagawa, S.; Koh, P.~W.; Hashimoto, T.~B.; and Liang, P. 2019.
\newblock Distributionally robust neural networks for group shifts: On the importance of regularization for worst-case generalization.
\newblock \emph{arXiv preprint arXiv:1911.08731}.

\bibitem[{Shankar et~al.(2018)Shankar, Piratla, Chakrabarti, Chaudhuri, Jyothi, and Sarawagi}]{shankar2018generalizing}
Shankar, S.; Piratla, V.; Chakrabarti, S.; Chaudhuri, S.; Jyothi, P.; and Sarawagi, S. 2018.
\newblock Generalizing Across Domains via Cross-Gradient Training.
\newblock In \emph{International Conference on Learning Representations}.

\bibitem[{Shui et~al.(2022)Shui, Chen, Wen, Zhou, Gagn{\'e}, and Wang}]{shui2022novel}
Shui, C.; Chen, Q.; Wen, J.; Zhou, F.; Gagn{\'e}, C.; and Wang, B. 2022.
\newblock A novel domain adaptation theory with Jensen--Shannon divergence.
\newblock \emph{Knowledge-Based Systems}, 257: 109808.

\bibitem[{Shui et~al.(2021)Shui, Li, Li, Gagn{\'e}, Ling, and Wang}]{shui2021aggregating}
Shui, C.; Li, Z.; Li, J.; Gagn{\'e}, C.; Ling, C.~X.; and Wang, B. 2021.
\newblock Aggregating From Multiple Target-Shifted Sources.
\newblock In \emph{Proceedings of the International Conference on Machine Learning}, 9638--9648.

\bibitem[{Shui, Wang, and Gagn{\'e}(2021)}]{shui2021benefits}
Shui, C.; Wang, B.; and Gagn{\'e}, C. 2021.
\newblock On the benefits of representation regularization in invariance based domain generalization.
\newblock \emph{arXiv preprint arXiv:2105.14529}.

\bibitem[{Snell, Swersky, and Zemel(2017)}]{proto}
Snell, J.; Swersky, K.; and Zemel, R.~S. 2017.
\newblock Prototypical Networks for Few-shot Learning.
\newblock \emph{CoRR}, abs/1703.05175.

\bibitem[{Sun and Saenko(2016)}]{sun2016deep}
Sun, B.; and Saenko, K. 2016.
\newblock Deep coral: Correlation alignment for deep domain adaptation.
\newblock In \emph{European conference on computer vision}, 443--450. Springer.

\bibitem[{Vapnik(1998)}]{vapnik1998Statistical}
Vapnik, V.~N. 1998.
\newblock \emph{Statistical Learning Theory}.
\newblock Bantam.

\bibitem[{Wainwright(2019)}]{wainwright2019high}
Wainwright, M.~J. 2019.
\newblock \emph{High-dimensional statistics: A non-asymptotic viewpoint}, volume~48.
\newblock Cambridge University Press.

\bibitem[{Wang, He, and Katabi(2020)}]{cida}
Wang, H.; He, H.; and Katabi, D. 2020.
\newblock Continuously Indexed Domain Adaptation.
\newblock In \emph{ICML}.

\bibitem[{Wang, Li, and Zhao(2022)}]{hanzhao_gda}
Wang, H.; Li, B.; and Zhao, H. 2022.
\newblock Understanding Gradual Domain Adaptation: Improved Analysis, Optimal Path and Beyond.
\newblock \emph{arXiv preprint arXiv:2204.08200}.

\bibitem[{Wang et~al.(2021)Wang, Lan, Liu, Ouyang, and Qin}]{jd_survey}
Wang, J.; Lan, C.; Liu, C.; Ouyang, Y.; and Qin, T. 2021.
\newblock Generalizing to Unseen Domains: {A} Survey on Domain Generalization.
\newblock \emph{CoRR}, abs/2103.03097.

\bibitem[{Wang et~al.(2023)Wang, Dong, Arik, and Yu}]{wang2023koopman}
Wang, R.; Dong, Y.; Arik, S.~O.; and Yu, R. 2023.
\newblock Koopman Neural Operator Forecaster for Time-series with Temporal Distributional Shifts.
\newblock In \emph{The Eleventh International Conference on Learning Representations}.

\bibitem[{Wen et~al.(2019)Wen, Liu, Zheng, Zheng, Gong, and Yuan}]{wen2019exploiting}
Wen, J.; Liu, R.; Zheng, N.; Zheng, Q.; Gong, Z.; and Yuan, J. 2019.
\newblock Exploiting local feature patterns for unsupervised domain adaptation.
\newblock In \emph{Proceedings of the AAAI Conference on Artificial Intelligence}, volume~33, 5401--5408.

\bibitem[{Xie et~al.(2018)Xie, Zheng, Chen, and Chen}]{xie2018learning}
Xie, S.; Zheng, Z.; Chen, L.; and Chen, C. 2018.
\newblock Learning semantic representations for unsupervised domain adaptation.
\newblock In \emph{International conference on machine learning}, 5423--5432. PMLR.

\bibitem[{Xu et~al.(2014)Xu, Li, Niu, and Xu}]{wjd_93}
Xu, Z.; Li, W.; Niu, L.; and Xu, D. 2014.
\newblock Exploiting low-rank structure from latent domains for domain generalization.
\newblock In \emph{European Conference on Computer Vision}, 628--643. Springer.

\bibitem[{Yan et~al.(2020)Yan, Song, Li, Zou, and Ren}]{yan2020improve}
Yan, S.; Song, H.; Li, N.; Zou, L.; and Ren, L. 2020.
\newblock Improve unsupervised domain adaptation with mixup training.
\newblock \emph{arXiv preprint arXiv:2001.00677}.

\bibitem[{Yue et~al.(2019)Yue, Zhang, Zhao, Sangiovanni-Vincentelli, Keutzer, and Gong}]{wjd_28}
Yue, X.; Zhang, Y.; Zhao, S.; Sangiovanni-Vincentelli, A.; Keutzer, K.; and Gong, B. 2019.
\newblock Domain randomization and pyramid consistency: Simulation-to-real generalization without accessing target domain data.
\newblock In \emph{Proceedings of the IEEE/CVF International Conference on Computer Vision}, 2100--2110.

\bibitem[{Zeng et~al.(2023)Zeng, Wang, Zhou, Ling, and Wang}]{zeng2023foresee}
Zeng, Q.; Wang, W.; Zhou, F.; Ling, C.; and Wang, B. 2023.
\newblock Foresee What You Will Learn: Data Augmentation for Domain Generalization in Non-stationary Environment.
\newblock \emph{Proceedings of the AAAI Conference on Artificial Intelligence}, 37(9): 11147--11155.

\bibitem[{Zheng et~al.(2018)Zheng, Wen, Liu, Long, Dai, and Gong}]{zheng2018unsupervised}
Zheng, N.; Wen, J.; Liu, R.; Long, L.; Dai, J.; and Gong, Z. 2018.
\newblock Unsupervised representation learning with long-term dynamics for skeleton based action recognition.
\newblock In \emph{Proceedings of the AAAI Conference on Artificial Intelligence}, volume~32.

\bibitem[{Zhou et~al.(2022)Zhou, Zhao, Zhang, Wang, Chang, Wang, and Zhu}]{zhou2022online}
Zhou, S.; Zhao, H.; Zhang, S.; Wang, L.; Chang, H.; Wang, Z.; and Zhu, W. 2022.
\newblock Online Continual Adaptation with Active Self-Training.
\newblock In \emph{International Conference on Artificial Intelligence and Statistics}, 8852--8883. PMLR.

\end{thebibliography}
\clearpage

\appendix
\onecolumn

% \section*{App}
\section{Theories}
\label{appendix_theory}

We first prove an intermediate lemma:

\begin{lemma}\label{lemma1}
Let $z\in\mathcal{Z}=\mathcal{X}\times\mathcal{Y}$ be the real-valued integrable random variable, let $P$ and $Q$ be two distributions on a common space $\mathcal{Z}$ such that $Q$ is absolutely continuous w.r.t. $P$. If for any function $f$ and $\lambda\in\R$ such that $\mathbb{E}_P[e^{\lambda(f(z)-\E_P(f(z))}]<\infty$, then we have:
\begin{equation*}
  \lambda (\E_{Q} f(z) - \E_{P}f(z)) \leq D_{\text{KL}}(Q\|P) + \log \E_P[e^{\lambda(f(z)-\E_P(f(z))}],
\end{equation*}
where $D_{\text{KL}}(Q\|P)$ is the Kullback–Leibler divergence between distribution $Q$ and $P$, and the equality arrives when $f(z)= \E_{P} f(z) + \frac{1}{\lambda}\log(\frac{d Q}{d P})$.
\end{lemma}
\begin{proof}
We let $g$ be {any} function such that $\E_P[e^{g(z)}]<\infty$, then we define a random variable $Z_g(z) = \frac{e^{g(z)}}{\E_P[e^{g(z)}]}$, then we can verify that $\E_{P}(Z_g) =1$. We assume another distribution $Q$ such that $Q$ (with distribution density $q(z)$) is absolutely continuous w.r.t. $P$ (with distribution density $p(z)$), then we have:
\begin{equation*}
\begin{split}
     \E_{Q}[\log Z_g] & = \E_{Q}[\log\frac{q(z)}{p(z)} + \log(Z_g\frac{p(z)}{q(z)})]  = D_{\text{KL}}(Q\|P) + \E_{Q}[\log(Z_g\frac{p(z)}{q(z)})]\\
     & \leq D_{\text{KL}}(Q\|P) + \log\E_{Q}[\frac{p(z)}{q(z)}Z_g]= D_{\text{KL}}(Q\|P) + \log \E_{P}[Z_g]
\end{split}
\end{equation*}
Since $\E_{P}[Z_g] = 1$ and according to the definition we have $\E_{Q}[\log Z_g] = \E_{Q}[g(z)] - \E_{Q}\log\E_{P}[e^{g(z)}] = \E_{Q}[g(z)] - \log\E_{P}[e^{g(z)}]$ (since $\E_{P}[e^{g(z)}]$ is a constant w.r.t. $Q$) and we therefore have:
\begin{equation}
    \E_{Q}[g(z)] \leq  \log\E_{P}[e^{g(z)}] + D_{\text{KL}}(Q\|P)
    \label{change_of_measure}
\end{equation}
Since this inequality holds for any function $g$ with finite moment generation function, then we let $g(z) = \lambda( f(z)-\E_P f(z))$ such that $\E_P[e^{f(z)-\E_P f(z)}]<\infty$. Therefore we have $\forall \lambda$ and $f$ we have:
\begin{equation*}
    \E_{Q}\lambda(f(z)-\E_P f(z)) \leq D_{\text{KL}}(Q\|P) + \log \E_P[e^{\lambda(f(z)-\E_P f(z)}]
\end{equation*}
Since we have $\E_{Q}\lambda(f(z)-\E_P  f(z)) = \lambda \E_{Q} (f(z)-\E_P f(z))) = \lambda (\E_{Q} f(z) - \E_{P} f(z))$, therefore we have:
\begin{equation*}
  \lambda (\E_{Q} f(z) - \E_{P} f(z)) \leq D_{\text{KL}}(Q\|P) + \log \E_P[e^{\lambda (\E_{Q} f(z) - \E_{P} f(z))}]
\end{equation*}
As for the attainment in the equality of Eq.~(\ref{change_of_measure}), we can simply set  $g(z) = \log(\frac{q(z)}{p(z)})$, then we can compute $\E_{P}[e^{g(z)}]=1$ and the equality arrives. Therefore in Lemma 1, the equality reaches when $\lambda(f(z)- \E_{P} f(z)) = \log(\frac{d Q}{d P})$.
\end{proof}
In the classification problem, we define the observation pair $z=(x,y)$. We also define the loss function $\ell(z)=L\circ h(z)$ with deterministic hypothesis $h$ and prediction loss function $L$. Then for abuse of notation, we simply denote the loss function $\ell(z)$ in this part.

Then we introduce the following bound between forecasted domain $\mathcal D_t^\mathcal{K}$ and real domain $\mathcal D_t$.
\begin{lemma}%[\cite{changjian_jsd}] 
\cite{shui2022novel}\label{lemmastart}
Let ${\mathcal D^\mathcal{K}_t}(h) = \mathcal{K}\circ \mathcal{G}\circ\phi(\mathcal{D}_m)$ be the \emph{forecasted} target domain, and suppose the loss function $\ell$  is bounded within an interval $G:G=\max(\ell)-\min(\ell)$. Then, for any $h \in \mathcal{H}$, its target risk $R_{\mathcal{D}_t}(h)$ can be upper bounded by:
\begin{align*}
    R_{\mathcal D_t}(h)\leq R_{\mathcal D^g_t}(h) +\frac{G}{\sqrt{2}}\sqrt{d_{KL}(\mathcal D_t||\mathcal D^\mathcal{K}_t)}.
\end{align*}
where we use $d_{}(\cdot||\cdot)$ to denote the KL divergence for simplification in the remaining paragraphs.
% \textbf{Remark}: The first term is the empirical risk on the domain $\mathcal D^g_t$. For the second term, Since $G$ is constant and only depends on loss function $l$, we can minimize it by matching distribution $\mathcal D_t^g$ and $\mathcal D_t$.
\end{lemma}
\begin{remark}
    {To achieve a low risk on $\mathcal{D}_t$, Lemma~\ref{lemmastart} suggests (1) learning $h$ and $g$ to minimize the risk over the forecasted domain $\mathcal{D}_t^\mathcal{K}$ and (2) learning $\mathcal{K}$ to minimize the KL divergence between $\mathcal{D}_t^\mathcal{K}$  and $\mathcal{D}_t$. While in practice $R_{\mathcal D^\mathcal{K}_t}(h)$ can be approximated by the empirical risk, Lemma~\ref{lemmastart}  still cannot provide any practical guidelines for learning $\mathcal{K}$ since $\mathcal{D}_t$ is unavailable. Moreover, note that $\mathcal{D}_t^\mathcal{K}$  can be replaced by $g(\mathcal{D}_i)$ for any other source domain $i$ in $\mathcal{E}$ and the bound still holds. Thus, Lemma~\ref{lemmastart} does not provide any theoretical insight into how to discover and leverage the evolving pattern in $\mathcal{E}$. }
\end{remark}
\begin{proof}
According to Lemma~\ref{lemma1}, $\forall \lambda>0$ we have: 
\begin{equation}
    \E_Q f(z) - \E_P f(z) \leq \frac{1}{\lambda} (\log\E_{P}~e^{[\lambda(f(z)-\E_{P}f(z))]} + D_{\text{KL}}(Q\|P))
    \label{sub_1}
\end{equation}

\noindent And $\forall \lambda<0$ we have:
\begin{equation}
    \E_Q f(z) - \E_P f(z) \geq \frac{1}{\lambda} (\log\E_{P}~e^{[\lambda(f(z)-\E_{P}f(z))]} + D_{\text{KL}}(Q\|P))
    \label{sub_2}
\end{equation}

% Then we introduce an intermediate distribution $\mathcal{M}(z) = \frac{1}{2}(\mathcal{D}(z) + \mathcal{D}'(z))$, then $\text{supp}(\mathcal{D})\subseteq\text{supp}(\mathcal{M})$ and $\text{supp}(\mathcal{D}')\subseteq\text{supp}(\mathcal{M})$, and 
Let $f=\ell$. Since the random variable $\ell$ is bounded through $G =  \max(\ell) -\min(\ell)$, then according to \cite{wainwright2019high} (Chapter 2.1.2), $\ell-\E_{P}\ell$ is sub-Gaussian with parameter at most $\sigma = \frac{G}{2}$, then we can apply Sub-Gaussian property to bound the $\log$ moment generation function:
\begin{equation*}
    \log\E_{P}~e^{[\lambda(\ell(z)-\E_{P}\ell(z))]} \leq \log e^{\frac{\lambda^2\sigma^2}{2}} \leq \frac{\lambda^2G^2}{8}.
\end{equation*}

\noindent In Eq.(\ref{sub_1}), we let $Q = \mathcal{D}'$ and $P=\mathcal{D}$, then $\forall \lambda>0$ we have:
\begin{equation}
    \E_{\mathcal{D}'}~\ell(z) - \E_{\mathcal{D}}~\ell(z) \leq \frac{G^2\lambda}{8} +  \frac{1}{\lambda}D_{\text{KL}}(\mathcal{D}'\|\mathcal{D})
    \label{sub_3}
\end{equation}

% \noindent In Eq.(\ref{sub_2}), we let $Q = \mathcal{D}$ and $P=\mathcal{M}$, then $\forall \lambda<0$ we have:
% \begin{equation}
%     \E_{\mathcal{D}}~\ell(z) - \E_{\mathcal{M}}~\ell(z) \geq \frac{G^2\lambda}{8} +  \frac{1}{\lambda}D_{\text{KL}}(\mathcal{D}\|\mathcal{M})
%     \label{sub_4}
% \end{equation}

% \noindent In Eq.(\ref{sub_3}), we denote $\lambda=\lambda_0>0$ and $\lambda=-\lambda_0<0$ in Eq.(\ref{sub_4}).
% Then Eq.(\ref{sub_3}),  Eq.(\ref{sub_4}) can be reformulated as:
% \begin{equation}
% \begin{split}
%     & \E_{\mathcal{D}'}~\ell(z) - \E_{\mathcal{M}}~\ell(z) \leq \frac{G^2\lambda_0}{8} +  \frac{1}{\lambda_0}D_{\text{KL}}(\mathcal{D}'\|\mathcal{M})\\
%     & \E_{\mathcal{M}}~\ell(z) - \E_{\mathcal{D}}~\ell(z) \leq \frac{G^2\lambda_0}{8} +  \frac{1}{\lambda_0}D_{\text{KL}}(\mathcal{D}\|\mathcal{M})
% \end{split}
%     \label{sub_5}
% \end{equation}
% Adding the two inequalities in Eq.(\ref{sub_5}), we therefore have:
% \begin{equation}
%     \E_{\mathcal{D}'}~\ell(z)  \leq \E_{\mathcal{D}}~\ell(z) + \frac{1}{\lambda_0} \big(D_{\text{KL}}(\mathcal{D}\|\mathcal{M}) + D_{\text{KL}}(\mathcal{D}'\|\mathcal{M}) \big) + \frac{\lambda_0}{4}G^2 
% \end{equation}
Since the inequality holds for $\forall \lambda$, then by taking $\lambda = \frac{2\sqrt{2}}{G}\sqrt{ D_{\text{KL}}(\mathcal{D}'\|\mathcal{D})}$ we finally have:
\begin{equation}
      \E_{\mathcal{D}'}~\ell(z)  \leq \E_{\mathcal{D}}~\ell(z) + \frac{G}{\sqrt{2}}\sqrt{D_{\text{KL}}(\mathcal{D}'\|\mathcal{D})}
    \label{sub_6}
\end{equation}
Let $\mathcal{D}' = \mathcal{D}_t$ and $\mathcal{D} = \mathcal{D}_t^\mathcal{K}$, we complete our proof.
\end{proof}

Given the definition of consistency, we can bound the target risk in terms of {\small $d_{}(\mathcal D^g_i||\mathcal D_i)$} in the source domains.

\begin{theorem}
\textbf{(Restatement of Theorem~\ref{theoremds})} Let $\{\mathcal D_1,\mathcal D_2,...,\mathcal D_m\}$ be $m$ observed source domains sampled sequentially from an evolving environment $\mathcal{E}$, and $\mathcal{D}_t$ be the next unseen target domain: $\mathcal{D}_t = \mathcal{D}_{m+1}$. Then, if $\mathcal{E}$ is $\lambda$-consistent, we have
{\small
\begin{align*}
  &R_{\mathcal{D}_t}(h) \le R_{\mathcal{D}_t^{\mathcal{K}^*}}(h)+\frac{G}{\sqrt{2(m-1)}}\Bigg(\sqrt{\sum _{i=2}^{m}d_{}(\mathcal D_{i}||\mathcal D_{i}^{\mathcal{K}^*})}+\sqrt{(m-1)\lambda}\Bigg).
\end{align*}
}
\end{theorem}
\begin{remark}\label{rm2}
{{\bf (1)}~Theorem~\ref{theoremds} highlights the role of the mapping function and $\lambda$-consistency in TDG. Given $\mathcal{K}^*$, the target risk $R_{\mathcal{D}_t}(h)$ can be upper bounded by in terms of loss on the forecasted target domain $R_{\mathcal{D}_t^{\mathcal{K}^*}}(h)$, $\lambda$, and the KL divergence between $\mathcal{D}_i$ and $\mathcal{D}_i^{\mathcal{K}^*}$ in all \emph{observed} source domains. When $\mathcal{K}^*$ can properly capture the evolving pattern of $\mathcal{E}$, we can train the classifier $h$ over the forecasted domain $\mathcal{D}_t^{\mathcal{K}^*}$ generated from $\mathcal{D}_m$ and can still expect a low risk on $\mathcal{D}_t$. {\bf(2)}~$\lambda$ is \emph{unobservable} and is determined by $\mathcal{E}$ and $\mathcal{K}$. Intuitively, a small $\lambda$ suggests high \emph{predictability} of $\mathcal{E}$, which indicates that it is easier to predict the target domain $\mathcal{D}_t$. On the other hand, a large $\lambda$ indicates that there does not exist a global mapping function that captures the evolving pattern consistently well over domains. Consequently, generalization to the target domain could be challenging and we cannot expect to learn a good hypothesis $h$ on $\mathcal{D}_t$. {\bf (3)}~In practice, $\mathcal{K}^*$ is not given, but can be learned by minimizing $d_{}(\mathcal D_{i}||\mathcal D_{i}^{\mathcal{K}})$ in source domains. Besides, aligning $D_{i}^{\mathcal{K}}$ and $\mathcal{D}_i$ is usually achieved by} representation learning: \emph{that is, learning $\mathcal{G}\circ\phi: \mathcal{X} \rightarrow \mathcal{Z}$ to minimize $d_{}(\mathcal D_{i}||\mathcal D_{i}^{\mathcal{K}}), \forall z \in \mathcal{Z}$.}
\end{remark}
\begin{proof}

According to Definition of $\lambda$-consistency, we have:

\begin{align*}
d_{}(\mathcal D_t||\mathcal D^{\mathcal{K}^*}_t)\leq d_{}(\mathcal D_2||\mathcal D^{\mathcal{K}^*}_2)+|d_{}(\mathcal D_t||\mathcal D^{\mathcal{K}^*}_t)-d_{}(\mathcal D_2||\mathcal D_2^{\mathcal{K}^*})|\leq d_{}(\mathcal D_2||\mathcal D_2^{\mathcal{K}^*})+\lambda
\end{align*}

Similarly, we have the followings:
\begin{align*}
&d_{}(\mathcal D_t||\mathcal D_t^{\mathcal{K}^*})\leq d_{}(\mathcal D_i||\mathcal D^{\mathcal{K}^*}_i)+|d_{}(\mathcal D_t||\mathcal D^{\mathcal{K}^*}_t)-d_{}(\mathcal D_i||\mathcal D_i^{\mathcal{K}^*})|\leq d_{}(\mathcal D_i||\mathcal D_i^{\mathcal{K}^*})+\lambda \\
& \hspace{180pt} \cdots\\
&d_{}(\mathcal D_t||\mathcal D_t^{\mathcal{K}^*})\leq d_{}(\mathcal D_m||\mathcal D^{\mathcal{K}^*}_m)+|d_{}(\mathcal D_t||\mathcal D_t^{\mathcal{K}^*})-d_{}(\mathcal D_m||\mathcal D_m^{\mathcal{K}^*})|\leq d_{}(\mathcal D_m||\mathcal D_m^{\mathcal{K}^*})+\lambda,
\end{align*}
% Sum over all, we have:
% \begin{equation*}
% (m-1)\cdot d_{}(\mathcal D^{\mathcal{K}^*}_t||\mathcal D_t)\leq \sum _{i=2}^{m}d_{}(\mathcal D_i^{\mathcal{K}^*}||\mathcal D_{i+1})+(m-1)\cdot\lambda
% \end{equation*}
which gives us
\begin{align*}
d_{}(\mathcal D_t||\mathcal D_t^{\mathcal{K}^*})\leq \frac{1}{m-1}\sum _{i=2}^{m}d_{}(\mathcal D_i||\mathcal D_{i}^{\mathcal{K}^*})+\lambda
\end{align*}
Then, according to Lemma~\ref{lemmastart}, we have
\begin{align*}
R_{\mathcal D_{t}}(h)&\leq R_{\mathcal D_t ^{\mathcal{K}^*}}(h)+\frac{G}{\sqrt{2}}\sqrt{d_{}(\mathcal D_t||\mathcal D^{\mathcal{K}^*}_t)}\\
&\leq R_{\mathcal D_t ^{\mathcal{K}^*}}(h)+\frac{G}{\sqrt{2}}\sqrt{\frac{1}{m-1}\sum _{i=2}^{m}d_{}(\mathcal D_i||\mathcal D_{i}^{\mathcal{K}^*})+\lambda}\\
& \le R_{\mathcal{D}_t^{\mathcal{K}^*}}(h)+\frac{G}{\sqrt{2(m-1)}}\Bigg(\sqrt{\sum _{i=2}^{m}d_{}(\mathcal D_{i}||\mathcal D_{i}^{\mathcal{K}^*})}+\sqrt{(m-1)\lambda}\Bigg)
\end{align*}
\end{proof}

% \begin{theorem}
% \label{theoremds}
% For any $g\in \mathcal G$ on a domain sequence $\{\mathcal D_1,\mathcal D_2,...,\mathcal D_m,\mathcal D_t\}$, we can calculate a group of generating discrepancies for each domain pair, denoted as $\{d_{}^1,d_{}^2,...,d_{}^m\}$. We define $\lambda(g)=\sup_{a,b}|d(g(\mathcal D_a),\mathcal D_{a+1})-d(g(\mathcal D_b),\mathcal D_{b+1})|$ as a measurement the variance of $\{d_{}^1,d_{}^2,...,d_{}^m\}$.

% $$d_{}(g(\mathcal D_m)||\mathcal D_t)\leq \frac{1}{m-1}\sum _{i=1}^{m-1}d_{}(g(\mathcal D_i)||\mathcal D_{i+1})+\lambda(g)$$

% \textbf{Remark}: The first term is average generating discrepancies of the source domain pairs. According to the assumption, the second term is a constant. This lead us to bound the target generating discrepancy with source generating discrepancies.
% \end{theorem}

% Although JS divergence is a great tool for measurement of distribution distances, it is challenging to align joint distribution of $\mathcal D(x,y)$ in practice. The previous work[cite] has provide upper bound in terms of marginal distribution of $y$ and conditional distribution of $x$ give $y$, which provide principled guidelines to developing a practical method. To motivate a more practical method, we decompose the joint distribution in $Theorem$ 1, then we have the next Corollary.
We first introduce the upper bound for Kullback–Leibler (KL) Divergence decomposition:
\begin{lemma}
\label{jsd_decompose}
Let $\mathcal D(z,y)$ and $\mathcal D'(z,y)$ be two distributions over $\mathcal X \times \mathcal Y$, $\mathcal D(y)$ and $\mathcal D'(y)$ be the corresponding marginal distribution of $y$, $\mathcal D(z|y)$ and $\mathcal D'(z|y)$ be the corresponding conditional distribution given $y$, then we can get the following bound,
\begin{align*}  
d_{}(\mathcal D(z,y)||\mathcal D'(z,y))&=d_{}(\mathcal D(y)||\mathcal D'(y))+\mathbb E_{y\sim \mathcal D(y)} d_{}(\mathcal D(z|y)||\mathcal D'(z|y))
\end{align*}
\end{lemma}

Based on Lemma \ref{jsd_decompose}, we can decompose $\mathcal{D}(z,y)$ into marginal and conditional distributions to motivate more practical TDG algorithms. For example, when it is decomposed into 
class prior $\mathcal{D}(y)$ and semantic  conditional distribution $\mathcal{D}(z|y)$, we have the following Corollary. 

\begin{corollary}
\label{corollay1}
Following the assumptions of Theorem 1, the target risk can be bounded by
{\small
\begin{align*}
 &R_{\mathcal D_t}(h) \leq  R_{\mathcal D_t^{\mathcal{K}^*}}(h)+\frac{G}{\sqrt{2(m-1)}}\Bigg(\underbrace{\sqrt{\sum _{i=2}^{m}d_{}(\mathcal D_{i}(y)||\mathcal D^{\mathcal{K}^*}_{i}(y))}}_{\bold{I}} + \sqrt{(m-1)\lambda}+\underbrace{\sqrt{\sum _{i=2}^{m}\mathbb E_{y\sim \mathcal D_{i}^{\mathcal{K}^*}(y)}d_{}(\mathcal D_{i}(z|y)||\mathcal D^{\mathcal{K}^*}_{i}(z|y))}}_{\bold{II}}\Bigg).
\end{align*}
}
\end{corollary}
\begin{remark}\label{rmk3}
{To generalize well to $\mathcal{D}_t$, Corollary~\ref{corollay1} suggests that a good mapping function should capture both label shifts (term I) and semantic conditional distribution shifts (terms II). In further, label shift can be eliminated by reweighting/resampling the instances according to the class ratios between domains~\cite{shui2021aggregating}. Then, we will have I = 0, and the upper bound can be simplified as Eq.~(\ref{eqn:bound}).
}\end{remark}

% \begin{lemma}\label{jsd-to-kl}
% Let $P$ and $Q$ be two distributions on a common space $\mathcal{Z}$, then the following inequality holds~\cite{lin1991divergence}:
% \begin{equation*}
%  \frac{1}{2} (D_{KL}(P||Q)+D_{KL}(Q||P))\ge d(P||Q)
% \end{equation*}
% where $D_{\text{KL}}(Q\|P)$ is the Kullback–Leibler divergence between distribution $Q$ and $P$, and $d_{}(P||Q)$ is the Jensen-Shannon (JS) divergence between $P$ and $Q$~\cite{jsd}..
% \end{lemma}

\begin{theorem}
\label{appendix:proof_of_loss_kl}
\textbf{(Restatement of Theorem~\ref{eqn:inequality})}
The optimization loss $J_{}$ defined in Eq.~(\ref{eqn:loss}), is the approximation of the upper bound of KL term in Def.~\ref{def1} and an inter-class distance loss, which implies

\begin{align*}
    \underbrace{- (K-1)\sum_{i=2}^{m}E_{y'\sim \mathcal D_{i}(y')} E_{ y\sim \mathcal 
    D_{i}^{\mathcal{K}^*}(y), 
         y\neq y'}  d(\mathcal{D}_{i+1}(z|y')||\mathcal{D}^{\mathcal{K}^*}_{i+1}(z|y))}_{\text{Inter-Class Distance loss}} 
    + \underbrace{(K-1)\sum_{i=2}^{m}\mathbb E_{y\sim \mathcal D_{i}(y)}  d_{}(\mathcal D_{i}(z|y)||\mathcal D^{\mathcal{K}^*}_{i}(z|y))}_{\text{KL-Divergence loss in Eq.~\ref{eqn:bound}}}
    \le J_{}
\end{align*}
\end{theorem}

\begin{proof}
{\small
    \begin{align}
J_{}&=\sum_{i=1}^{m-1}\sum_{k=1}^K\sum_{n=1}^{N_B}-\frac{1}{KN_B}\log \mathcal{D}_{}(y_{n}^{i+1}=k|x_n^{i+1})\nonumber\\
&=\sum_{i=1}^{m-1}\sum_{k=1}^K\sum_{n=1}^{N_B}-\frac{1}{KN_B}\log\frac{\exp{-d_{\text{eu}}(\mathcal{G}\circ\phi(x_n^{i+1}),c^k_i)}}{\sum_{k'=1}^K\exp{-d_{\text{eu}}(\mathcal{G}\circ\phi(x_n^{i+1}),c^{k'}_i)}}\nonumber\\
&=\sum_{i=1}^{m-1}\sum_{k=1}^K\sum_{n=1}^{N_B}\frac{1}{KN_B}-\log\frac{1}{\sum_{k'=1}^K \frac{\exp{-d_{\text{eu}}(\mathcal{G}\circ\phi(x_n^{i+1}),c^{k'}_i)}}{\exp{-d_{\text{eu}}(\mathcal{G}\circ\phi(x_n^{i+1}),c^k_i)}}}\nonumber\\
&=\sum_{i=1}^{m-1}\sum_{k=1}^K\sum_{n=1}^{N_B}\frac{1}{KN_B}\log\sum_{k'=1}^K \frac{\exp{-d_{\text{eu}}(\mathcal{G}\circ\phi(x_n^{i+1}),c^{k'}_i)}}{\exp{-d_{\text{eu}}(\mathcal{G}\circ\phi(x_n^{i+1}),c^k_i)}}\nonumber \\
&\ge \sum_{i=1}^{m-1}\sum_{k=1}^K\sum_{n=1}^{N_B}\frac{1}{KN_B} \sum_{k'=1}^K\log\frac{\exp{-d_{\text{eu}}(\mathcal{G}\circ\phi(x_n^{i+1}),c^{k'}_i)}}{\exp{-d_{\text{eu}}(\mathcal{G}\circ\phi(x_n^{i+1}),c^k_i)}} \label{eq:jensen}\\
&=\sum_{i=1}^{m-1}\sum_{k=1}^K\sum_{n=1}^{N_B}\frac{1}{KN_B} \big(\sum_{k'=1}^K -d_{\text{eu}}(\mathcal{G}\circ\phi(x_n^{i+1}),c^{k'}_i) + K\cdot d_{\text{eu}}(\mathcal{G}\circ\phi(x_n^{i+1}),c^k_i)\big)\nonumber \\
&=\sum_{i=1}^{m-1}\sum_{k=1}^K\sum_{n=1}^{N_B}\frac{1}{KN_B} \big(\sum_{k'=1}^K - (\mathcal{G}\circ\phi(x_n^{i+1})-c_i^{k'})^2 + K (\mathcal{G}\circ\phi(x_n^{i+1})-c_i^{k})^2\big) \label{eq:euclidean}\\
&=\sum_{i=1}^{m-1}\sum_{k=1}^K\sum_{n=1}^{N_B}\frac{1}{KN_B} \big(\sum_{k'=1}^K - (\mathcal{G}\circ\phi(x_n^{i+1})-c_i^{k'})^2 + K (\mathcal{G}\circ\phi(x_n^{i+1})-c_i^{k})^2\big)\nonumber\\
&=\sum_{i=1}^{m-1}\sum_{k=1}^K\frac{1}{K} \big(\sum_{k'=1}^K - (\sum_{n=1}^{N_B}\frac{\mathcal{G}\circ\phi(x_n^{i+1})}{N_B}-c_i^{k'})^2 + K (\sum_{n=1}^{N_B}\frac{\mathcal{G}\circ\phi(x_n^{i+1})}{N_B}-c_i^{k})^2\big)\nonumber \\
&=\sum_{i=1}^{m-1}\sum_{k=1}^K\frac{1}{K} \big(\sum_{k'=1 ,k'\not= k}^K - (\hat{\mu}^{i+1}_k-c_i^{k'})^2 + (K-1) (\hat{\mu}^{i+1}_k-c_i^{k})^2\big)\label{eq:estimate}
\\
\end{align}
\begin{align}
&=- \sum_{i=1}^{m-1}\sum_{k=1}^K   \frac{1}{K} \sum_{ \tiny
 k'=1,
  k'\not= k }^Kd(\mathcal{D}_{i+1}(z|y=k')||\mathcal{D}^{\mathcal{K}^*}_{i+1}(z|y=k))
    + \sum_{i=1}^{m-1}\sum_{k=1}^K \frac{K-1}{K}d_{}(\mathcal D_{i+1}(z|y=k)||\mathcal D^{\mathcal{K}^*}_{i+1}(z|y=k))\label{eq:to_kl}\\
&=- (K-1)\sum_{i=1}^{m-1}\mathbb E_{y'\sim \mathcal D_{i}(y')} E_{ y\sim \mathcal 
    D_{i}^{\mathcal{K}^*}(y), 
         y\neq y'}  d(\mathcal{D}_{i+1}(z|y')||\mathcal{D}^{\mathcal{K}^*}_{i+1}(z|y))
    + (K-1)\sum_{i=1}^{m-1}\mathbb E_{y\sim \mathcal D_{i}(y)}  d_{}(\mathcal D_{i+1}(z|y)||\mathcal D^{\mathcal{K}^*}_{i+1}(z|y))\label{eq:avg_exp}\\
&= - (K-1)\sum_{i=2}^{m}E_{y'\sim \mathcal D_{i}(y')} E_{ y\sim \mathcal 
D_{i}^{\mathcal{K}^*}(y), 
     y\neq y'}  d(\mathcal{D}_{i}(z|y')||\mathcal{D}^{\mathcal{K}^*}_{i}(z|y))
+ (K-1)\sum_{i=2}^{m}\mathbb E_{y\sim \mathcal D_{i}(y)}  d_{}(\mathcal D_{i}(z|y)||\mathcal D^{\mathcal{K}^*}_{i}(z|y))\nonumber
\end{align}
}

In Eq.~(\ref{eq:jensen}), we apply Jensen Inequality; In Eq.~(\ref{eq:euclidean}), we replace $d_{\text{eu}}(\cdot,\cdot)$ with the euclidean distance; In Eq.~(\ref{eq:estimate}), we define $\hat{\mu}^{i+1}_k=\sum_{n=1}^{N_B}\frac{\mathcal{G}\circ\phi(x_n^{i+1})}{N_B}$, which is the empirical estimation of $\mu_{x^{i+1}|y^{i+1}=k}=\E[x^{i+1}|y^{i+1}=k]$; In Eq.~(\ref{eq:to_kl}), since we can approximate the conditional distribution on $X$ as Gaussian distribution with identical Covariance in Section 4 of~\cite{shui2021aggregating}, KL distance $d(D_{i+1}(z|y=k')||\mathcal{D}^{\mathcal{K}^*}_{i+1}(z|y=k))$ can be approximated by $(c_i^{k'}-\hat{\mu}^{i+1}_k)^2+\text{const}$  and $d_{}(\mathcal D_{i+1}(z|y=k)||\mathcal D^{\mathcal{K}^*}_{i+1}(z|y=k))$ can be approximated by $(c_i^{k}-\hat{\mu}^{i+1}_k)^2+\text{const}$. The constants can be canceled out with the subtractions between the two terms, which also does not affect the optimization objective with the replacement. In Eq.~(\ref{eq:avg_exp}), we use average to approximate expectations w.r.t. $y$ and $y'$. Hence, we conclude the proof.

% Also, based on Lemma~\ref{jsd-to-kl}, we can have
% {\small
% \begin{align}
%      d(\mathcal{D}^{\mathcal{K}^*}_{i+1}(x|y=k')||\mathcal{D}_{i+1}(x|y=k))\le \frac{1}{2} (d_{KL}(\mathcal{D}^{\mathcal{K}^*}_{i+1}(x|y=k')||\mathcal{D}_{i+1}(x|y=k))+d_{KL}(\mathcal{D}_{i+1}(x|y=k)||\mathcal{D}^{\mathcal{K}^*}_{i+1}(x|y=k'))) \label{eq:js inequality1}
% \end{align}
% }
% {\small
% \begin{align}
%      d_{}(\mathcal D^{\mathcal{K}^*}_{i+1}(x|y=k)||\mathcal D_{i+1}(x|y=k))\le \frac{1}{2} (d_{KL}(\mathcal D^{\mathcal{K}^*}_{i+1}(x|y=k)||\mathcal D_{i+1}(x|y=k))+d_{KL}(\mathcal D_{i+1}(x|y=k)||\mathcal D^{\mathcal{K}^*}_{i+1}(x|y=k))) \label{eq:js inequality2}
% \end{align}
% }

\end{proof}

\section{Additional Experiments}
% Domainbed

\subsection{Further Investigation of Interpolation and Extrapolation}
\label{appendix:inter_extra}

In Table \ref{tab:domain_num_RMNIST_table}, we can observe that the performances of ERM are not improved as the number of domains increases, which is counter-intuitive. We speculate that this is due to the ``extrapolation" nature of TDG.  To further investigate its impact on the generalization performance on the target domain, we compare the following three settings on the RMNIST dataset: 

\begin{enumerate}
  \item \textbf{TKNets-Evolving (Extrapolation).} Same as TKNets in Section~\ref{sec:exp}, where the target domain $\mathcal{D}_t = \mathcal{D}_{i+1}$.
  % The environment setting is same as (2) but with algorithm replaced by DPNets. Note that (2) and (3) share the same environments setting but only with different algorithms. 
  \item \textbf{ERM-Evolving (Extrapolation).} Same as ERM in Section~\ref{sec:exp}, where the target domain $\mathcal{D}_t = \mathcal{D}_{i+1}$.
  \item \textbf{ERM-Interpolation.} The ERM approach uses the domain in the middle as the target domain, and the rest domains as the source domains. 
\end{enumerate}

Note that (1) and (2) are different algorithms with the same problem setup, and (2) and (3) use the same algorithm but with different problem setups.

We vary the number of domain and domain distances, and the results are reported in Fig.~\ref{erm_distance}, from which we have the following observations:

\begin{enumerate}
    \item The overall trend of the TKNets is going up as the number of domains increases. 
    \item The performances of ERM-Evolving do not increase as a function of the number of domain distances, which is consistent with the results in Table~\ref{tab:domain_num_RMNIST_table}.
    \item The performances of ERM-Interpolation increase as a function of the number of domain distance
    \item The improvements of  TKNets and ERM-Interpolation are not obvious once having a sufficient amount of domains (e.g., \# of domains = 7 for ERM-Interpolation). We conjecture that it is because the evolving pattern of RMNIST is relatively simple. Thus, a small number of domains are sufficient to learn such a pattern, and increasing the number of domains may not necessarily improve the performances of TKNets and  ERM-Interpolation anymore. 
\end{enumerate}

The results indicate for extrapolation, having more domains does not necessarily help learn an invariant representation if we do not leverage the evolving pattern. Intuitively, as the target domain is on the ``edge" of the domains, having more domains also indicates that it is further away from the “center” of the source domains, which may even make the generalization even more challenging. On the other hand, if the target domain is ``among" the source domains (i.e., when we perform ``interpolation"), the source domains may act as augmented data which improve the generalization performance. In other words, if the more source domains will be beneficial for ``interpolation" but not necessarily for ``extrapolation" if the evolving pattern is not properly exploited.

% more domains bring more samples, so that one can expect that the accuracy of a DG model increases monotonically as the number of domains becomes larger. While in Table \ref{tab:domain_num_RMNIST_table}, we can see that ERM does not behave in the same way under our problem setup. This motivate us to further investigate relationship between performance of ERM/DPNets and number of domains. The comparison is performed under the following three settings:

% \begin{enumerate}
%   \item \textbf{ERM-Interpolation.} The target domains is always $\{0^\circ\}$ and source domains are generated symmetrically on both sides of the target domain. 
%   \item \textbf{ERM-Evolving (Extrapolation).} The target domains is always $\{0^\circ\}$ and source domains are generated on one side of source domain. For example, with same setting in (1), the training domains would be $\{5^\circ, 10^\circ, 15^\circ,20^\circ, 25^\circ, 30^\circ\}$.
%   \item \textbf{DPNets-Evolving (Extrapolation).} The environment setting is same as (2) but with algorithm replaced by DPNets. Note that (2) and (3) share the same environments setting but only with different algorithms. 
% \end{enumerate}

\begin{figure*}[htbp]
		\centering 
		\subfloat[Distance = $3^\circ$]{\includegraphics[width=0.25\textwidth]{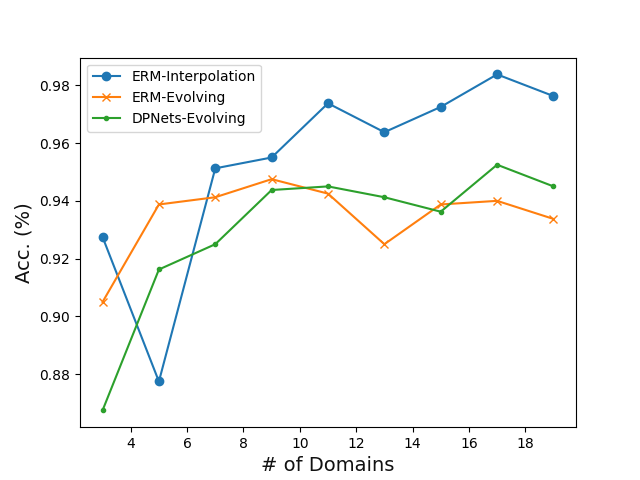}\label{erm_distance:3}}
		\subfloat[Distance = $7^\circ$]{\includegraphics[width=0.25\textwidth]{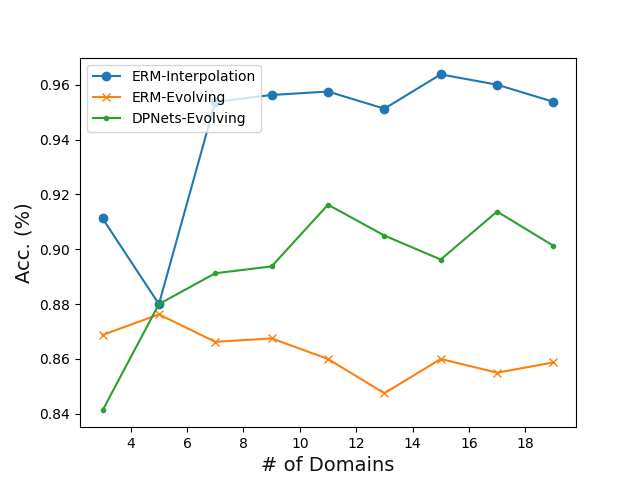}\label{erm_distance:7}}
		\subfloat[Distance = $11^\circ$]{\includegraphics[width=0.25\textwidth]{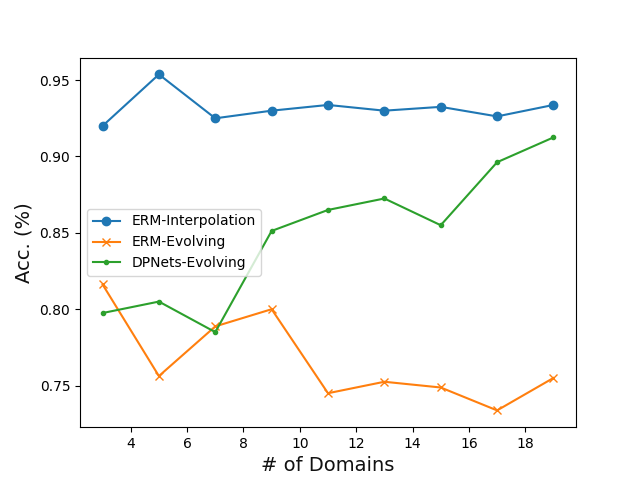}\label{erm_distance:11}}
		\subfloat[Average]{\includegraphics[width=0.25\textwidth]{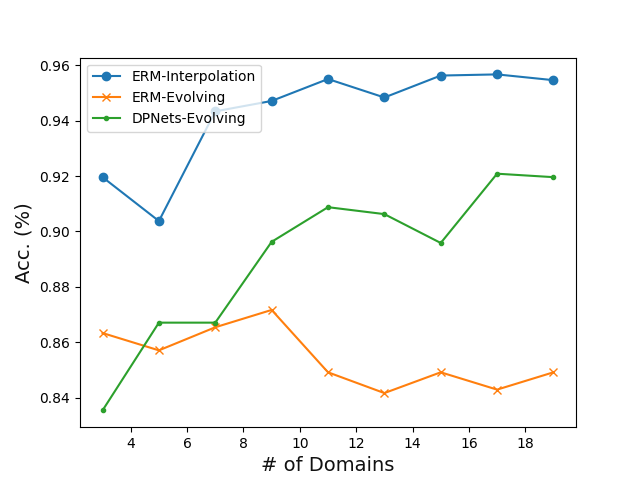}\label{erm_distance:ave}}
\caption{Performance of algorithms when numbers of domains changes.}
\label{erm_distance}
\end{figure*}

\subsection{Incorporating Domain Information into ERM.}
\label{appendix:domain_info}

The ERM in Section~\ref{sec:exp} does not leverage the index information of the source domains. In order to make a more fair comparison, we incorporate the index information into ERM. Specifically, we investigate three strategies for incorporating the index information used in the literature \cite{liu2020heterogeneous, long2017conditional,  zheng2018unsupervised, wen2019exploiting}: (1) Index Concatenation (Fig. \ref{index_info:one}), where the domain index is directly concatenated as a one-dimension feature \cite{li2021learning}; (2) One-hot Concatenation (Fig. \ref{index_info:two}), where the domain index is first one-hot encoded \cite{liu2020heterogeneous} and then concatenated to the original features \cite{long2017conditional}; (3) Outer product (Fig. \ref{index_info:three}), were flattened the outer product of original features and the one-hot indexes are used as the final input \cite{shui2021benefits}.

% One may argue that the lack of domain index information in the baselines make it an unfair comparison. In this section, we investigate the ability of incorporating domain index information for traditional DG algorithms on tabular datasets. Assume that there are $m+1$ domains, where the last one is target domain. For the experiments, we use three different incorporating strategies including: (1) One-dimension expanding (Fig. \ref{index_info:one}), where the domain index is added as a one-dimension feature \citep{li2021learning}; (2) One-hot Expanding (Fig. \ref{index_info:two}), where the index is expanded to $m+1$ dimensions one-hot format and then be concatenated with the original features \citep{long2017conditional}; (3) Outer product (Fig. \ref{index_info:three}), where the outer product of original features and the one-hot indexes is used as the the final input \citep{shui2021benefits}. To the best of our knowledge, this is the most straightforward and effective way of adding features in machine learning. Note that all the final features mentioned above are flattened to one dimension to make the network structure consistent during the comparison.

\begin{figure*}[htbp]
		\centering 
		\subfloat[Index Concatenation ]{\includegraphics[width=0.3\textwidth]{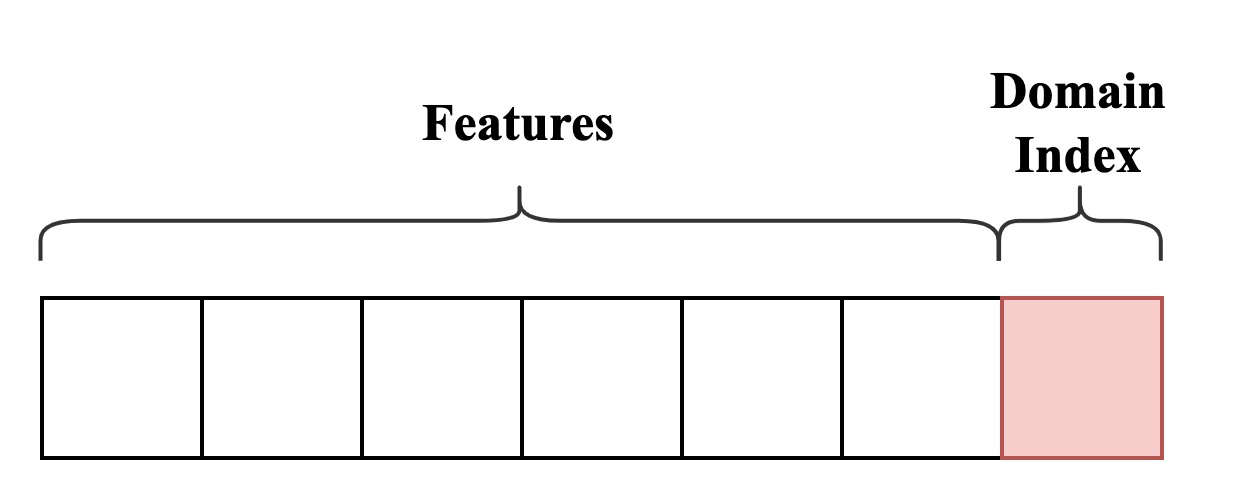}\label{index_info:one}}
		\subfloat[One-hot Concatenation]{\includegraphics[width=0.3\textwidth]{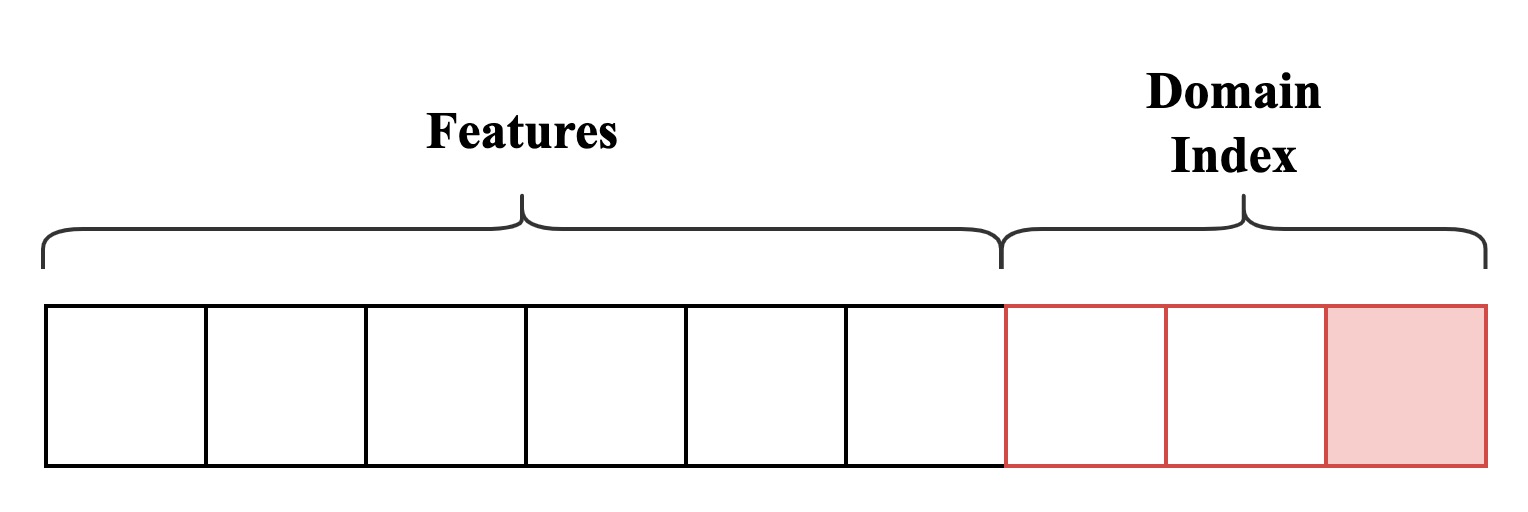}\label{index_info:two}}
		\subfloat[Outer Product]{\includegraphics[width=0.3\textwidth]{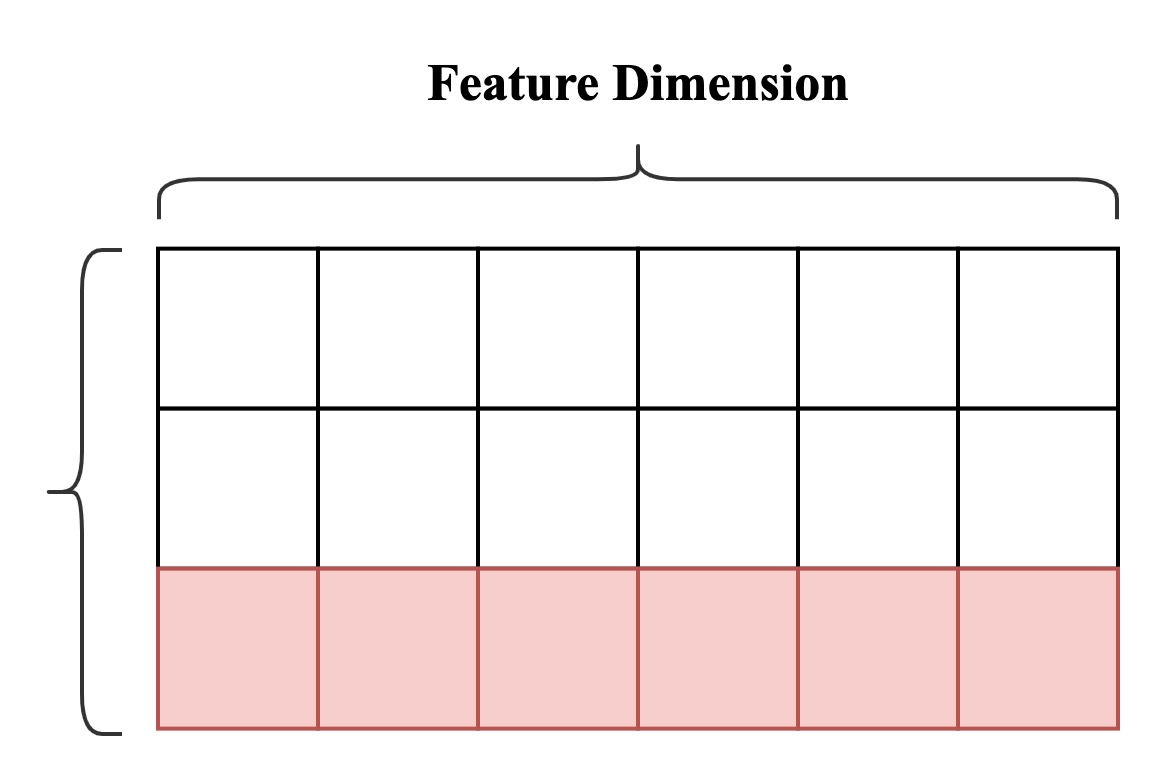}\label{index_info:three}}
\caption{Three domain index information incorporation strategies.}
\label{index_info}
\end{figure*}

We evaluate the algorithms on the EvolCircle and RPlate datasets and the results are reported in Table~\ref{tab:domain_index}. The experimental results verify the advantage of our algorithm in exploiting evolving information. We can observe that the improvements induced by incorporating the domain index are marginal, which indicates that it cannot properly leverage the evolving pattern of the environment.
 
% As we can see from the table, the performance of DG algorithms is not good enough even with domain index information incorporated. The results indicate that traditional DG algorithms can not make use of domain index information well. Adding domain information directly to the data does not significantly improve the performance of the algorithm. Instead, we need to design specialized structures to exploit the domain information. The performance improvement of our DPNets demonstrates that it can leverage the evolving pattern better.
\begin{figure}[t]
		\centering 
		\subfloat[ERM]{\includegraphics[width=0.32\textwidth]{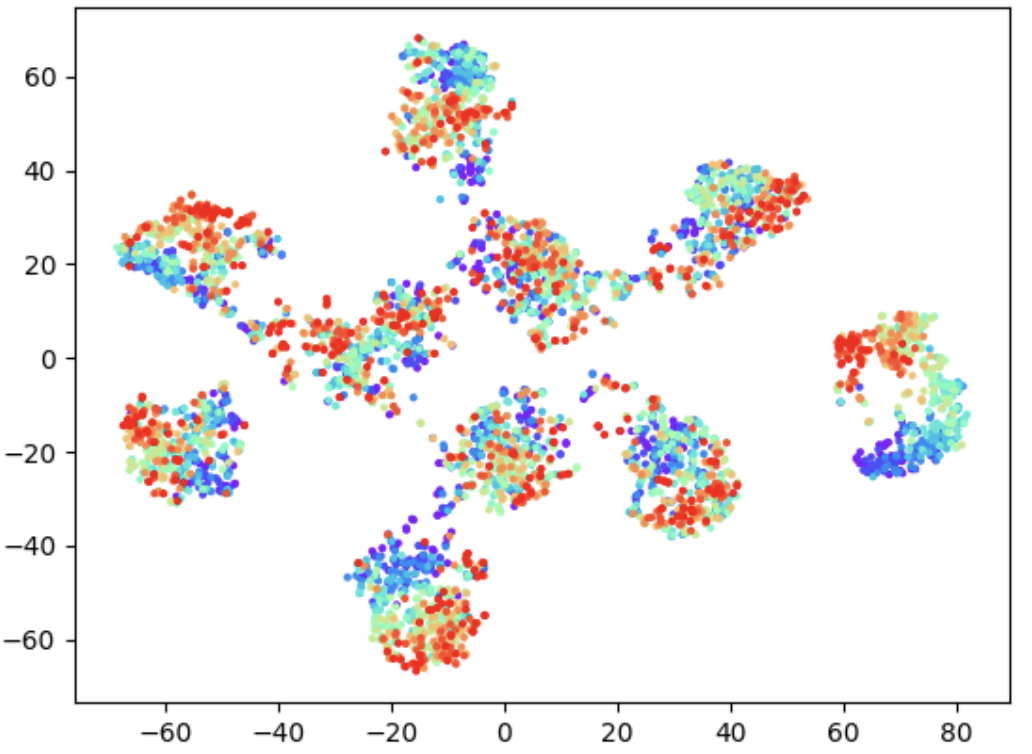}}
		\subfloat[TKNets]{\includegraphics[width=0.32\textwidth]{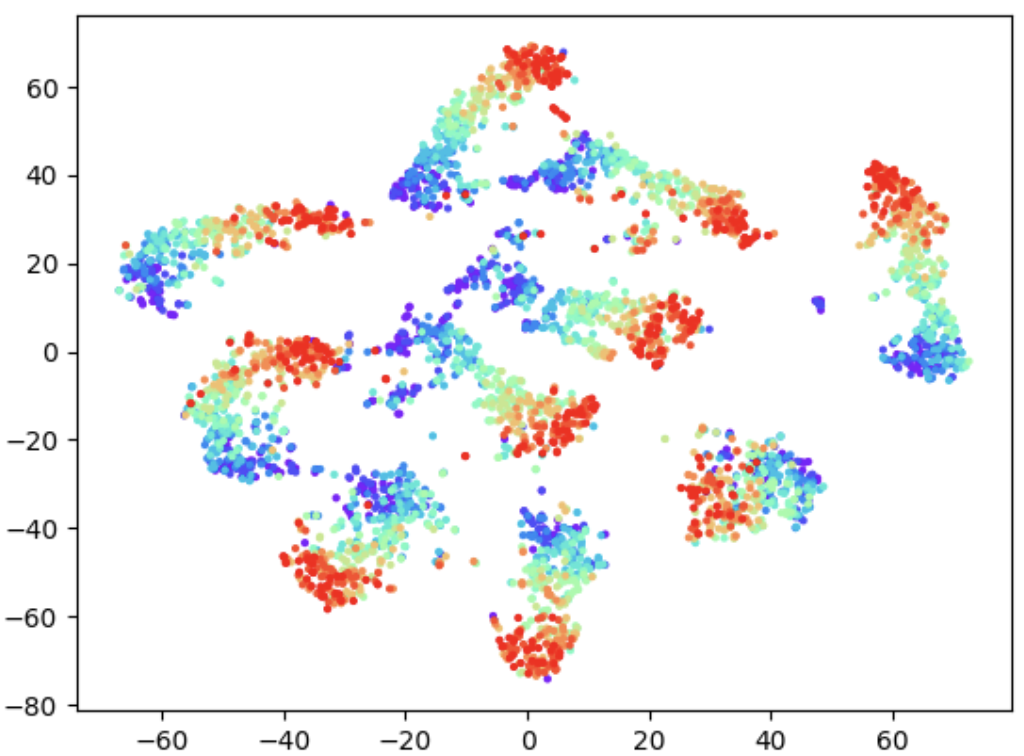}}
\caption{T-SNE Visualization of embedded RMNIST data.}
\label{tsne}
\end{figure}

\begin{table}[t]
    \caption{Performance of the traditional DG algorithms with domain index information incorporated.}
    \label{tab:domain_index}
    \begin{center}
    \adjustbox{max width=0.45\textwidth}{%
    \begin{tabular}{lccccccc}
    \toprule
    \textbf{Strategy}         & \textbf{EvolCircle}  & \textbf{RPlate} & \textbf{Average}              \\
    \midrule
    ERM       &72.7 $\pm$ 1.1              &63.9 $\pm$ 0.9              & 68.3                                  \\
    ERM + One-Dimension        &73.6 $\pm$ 0.6              &64.9 $\pm$ 0.8              & 69.3                                 \\
    ERM + One-Hot              &74.6 $\pm$ 0.3              &64.0 $\pm$ 0.3               & 69.3                                  \\
    ERM + Outer Product        &74.6 $\pm$ 0.4              &65.3 $\pm$ 0.2              & 70.0                                  \\
    TKNets (Ours)              & \textbf{94.2 $\pm$ 0.9}& \textbf{95.0 $\pm$ 0.5}& \textbf{94.6}                                  \\
    \bottomrule
    \end{tabular}}
    \end{center}
    \end{table}

\subsection{Ablation Study}
\textbf{The Impact of the Choice of Predefined Measurement Functions and Learned Measurement Functions} Predefined measurement functions can be polynomial functions, exponential functions, sine functions, etc. The learned measurement functions are trainable nonlinear multi-layer perceptrons. \cite{wang2023koopman} shows the predefined measurement functions outperform the model with learned measurement functions in time-series forecasting tasks. However, in TDG tasks, predefined measurement functions and learned measurement functions can outperform the others in different datasets in Table~\ref{Table:meas func}. Hence, we suggest the choice of the type of measurement function can be a hyperparameter depending on the dataset.
 % \textbf{A Prototype-based Classifier Intergrated with Other Baselines.} To explore whether the naive Prototype Network (Proto) will improve without learning evolving patterns, we combine the Prototypical Networks with other baselines in Table~\ref{Table:cos baseline}. DPNets outperform baselines integrated with the prototype classifier, which indicates the Prototypical method alone can not achieve comparable performance without leveraging the evolving patterns. 

\begin{table}[t!]
\caption{Performance of Predefined / Learned Measurement Functions }
\label{Table:meas func}
\setlength{\tabcolsep}{3.5pt}
% \vskip -0.1in
% \vspace{-0.5cm}
\begin{center}
\begin{sc}
\scalebox{0.9}{
\adjustbox{max width=0.42\textwidth}{\begin{tabular}{cccc}
\toprule
% \textbf{Algorithm} & P & RM  \\
% \midrule
% MLDG + GAN & 92.4 & 85.4\\
% MixUp & 92.3 & 85.7\\
% DAML & 92.7 & 86.4\\
% Ours  & 94.6 & 88.4\\

Dataset& RMNIST & Portrait & CoverType \\
\midrule
Predefined & 86.6 $\pm$ 0.2  & \textbf{97.2 $\pm$ 0.0} & \textbf{73.8 $\pm$ 1.0  }   \\
Learned  & \textbf{87.5 $\pm$ 0.1} &  96.4 $\pm$ 0.0 & 72.5 $\pm$ 0.8\\

% Portrait & 96.4 $\pm$ 0.0 & 95.6 $\pm$ 0.9 & 96.3 $\pm$ 0.3 & 96.5 $\pm$ 0.2 &  95.7 $\pm$ 0.3\\
% Rotating MNIST & 87.5 $\pm$ 0.1 &  82.5 $\pm$ 1.7 & 85.6 $\pm$ 1.2 & 85.4 $\pm$ 0.6 & 84.6 $\pm$ 0.4\\
% \midrule
\bottomrule
\end{tabular}}}
\end{sc}
\end{center}
\end{table}

\section{T-SNE visualization over evolving domains.} 
In this part, we investigate the ability of TKNets and DG methods in distilling domain-evolving information. Learning an invariant representation across all domains is a common motivation of domain generalization approaches. While in the TDG scenario, we need to leverage the evolving pattern to improve the generalization process. Here, we use t-SNE to visualize, respectively, the representations learned from the second-to-last layer by ERM and TKNets in Fig. \ref{tsne}. The colors from red to blue correspond to the domain's index from 1 to 30. The feature visualizations demonstrate that TKNets can keep the domain evolving even in the last layer of the network, which makes it possible to leverage that knowledge. On the contrary, the evolving pattern learned by the ERM is less obvious. The results further prove that ERM can not leverage evolving information well.

\section{Implementation Details}
\label{appendix:implement detail}
%####

We implement our algorithm based on \cite{domainbed}. To justify algorithm comparison between baselines and our algorithm, we adopted a random search of 20 trials for the hyper-parameter distribution. For each parameter combination, 5 repeated experiments are conducted. Then, we report the highest average performance for each algorithm-dataset pair. In this way, all parameters are automatically selected without human intervention, making the comparison of experimental results of different algorithms on different data fair and reliable. Almost all backbone and setting are following \cite{domainbed} except the following. In one single experiment, the model structure of $
\phi$ keeps the same. For EvolCircle and RPlate, we only use one single-layer network to make the classifier linear for all algorithms. For other datasets, networks are randomly chosen based on the random search algorithm. For Portrait, RMNIST, and Cover Type datasets, the learning rate is uniformly sampled from $[10^{-5},10^{-3.5}]$. The learning rate of the other datasets is uniformly sampled from $[10^{-4.5},10^{-2.5}]$. Experiments are conducted under at most 32x P100 parallelly.

\newpage

\section{TKNets Illustrative Diagram}
\begin{figure*}[h!]
		\centering 
%虚线的颜色. 
% K trajectories. binary
  \includegraphics[width=0.65\textwidth]{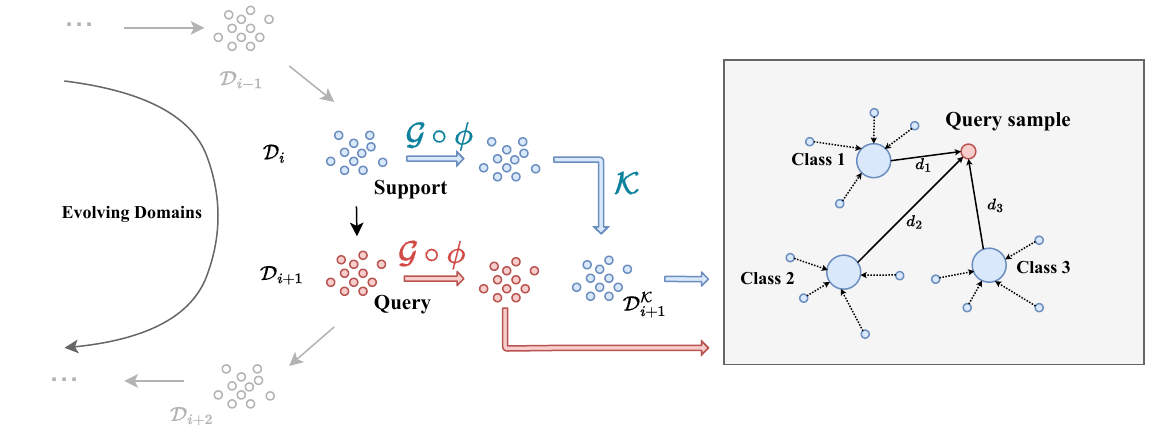}
   \caption{Data from \(\mathcal{D}_i\) and \(\mathcal{D}_{i+1}\) is mapped into the Koopman Space by the encoder \(\phi\) and predefine nonlinear measurement functions \(\mathcal{G}\) sequentially. Subsequently, the Koopman Space Representations \(\mathcal{D}_{i+1}^{\mathcal{K}}\) is constructed using the Koopman Operator \(\mathcal{K}\). The process of aligning \(\mathcal{D}_{i+1}(z|y)\) and \(\mathcal{D}_{i+1}^{\mathcal{K}}(z|y)\) suggested by Theorem 4.2, is achieved by minimizing the distance between samples from class \(k\) in domain \(i+1\) with the \(k\)-th class centroid calculated from \(\mathcal{D}_{i+1}^{\mathcal{K}}\), as proven in Theorem 4.3.}
\end{figure*}

\end{document}